\documentclass[twoside]{article}

\usepackage[preprint]{aistats2026}

\usepackage{times}  
\usepackage{helvet}  
\usepackage{courier}  
\usepackage[hyphens]{url}  
\usepackage{graphicx} 
\usepackage{caption} 
\usepackage{algorithm}
\usepackage{algorithmic}
\usepackage{booktabs} 
\usepackage{bm}

\usepackage{amsthm}
\newtheorem{defi}{Definition}
\newtheorem{theo}{Theorem}
\newtheorem{lemma}{Lemma}

\usepackage{threeparttable}
\usepackage{multirow}

\graphicspath{{figs/}}
\usepackage{amsmath} 

\usepackage{microtype}  

\usepackage[round]{natbib}

\begin{document}

%
%

\twocolumn[

\aistatstitle{CATS-Linear: Classification Auxiliary Linear Model for Time Series Forecasting}

\aistatsauthor{ Zipo Jibao, \quad Yingyi Fu \And  Xinyang Chen \And Guoting Chen}

\aistatsaddress{\{23s058023, 23s058019\}@stu.hit.edu.cn \\ School of Science, \\ Harbin Institute of Technology, \\ Shenzhen, China \And 21b358001@stu.hit.edu.cn \\ Harbin Institute of Technology, \\ Shenzhen, China \\ and University of Lille, France \And guoting.chen@univ-lille.fr \\ School of Science, \\ Great Bay University, \\ Dongguan, China } ]

\begin{abstract}
    Recent research demonstrates that linear models achieve forecasting performance competitive with complex architectures, yet methodologies for enhancing linear models remain underexplored. Motivated by the hypothesis that distinct time series instances may follow heterogeneous linear mappings, we propose the \textbf{C}lassification \textbf{A}uxiliary \textbf{T}rend-\textbf{S}easonal Decoupling \textbf{Linear} Model \textbf{CATS-Linear}, employing Classification Auxiliary Channel-Independence (CACI). CACI dynamically routes instances to dedicated predictors via classification, enabling supervised channel design. We further analyze the theoretical expected risks of different channel settings. Additionally, we redesign the trend-seasonal decomposition architecture by adding a decoupling---linear mapping---recoupling framework for trend components and complex-domain linear projections for seasonal components. Extensive experiments validate that CATS-Linear with fixed hyperparameters achieves state-of-the-art accuracy comparable to hyperparameter-tuned baselines while delivering SOTA accuracy against fixed-hyperparameter counterparts.
\end{abstract}

\section{INTRODUCTION}

Multivariate time series forecasting uses the lookback window $x \in R^{D\times L}$ to predict the horizon window $y \in R^{D\times H}$, where $D$ is the number of features in the series and $L/H$ is the lookback/horizon window size. Channel-mixing (CM) \citep{18, 23mts, han24}, which is also referred to as channel-dependence \citep{han24}, models all features of $x$ as a whole to forecast, while channel-independence (CI) \citep{16, 27} trains $D$ independent models for each feature and only utilizes the $i$-th feature of $x$ to predict the $i$-th feature of $y$, as illustrated in Figure \ref{intro-fig1}. One-channel (OC) method \citep{nbeats20, 16, xu24}, which is called the global univariate method in \citep{2023tide}, ignores the features' differences and trains a single univariate forecasting model for all different features. 

Channel-mixing learns dependency relationships between different features, which inherently leads to a large parameter count. This configuration may induce overfitting and consequently degrade performance \citep{liu24}. Channel-independent methods are more robust; however, when applied to high-dimensional data, they require training numerous separate models, resulting in substantial memory consumption.


\begin{figure}[ht]\rmfamily      
\begin{center}
\centerline{\includegraphics[width=0.95\columnwidth]{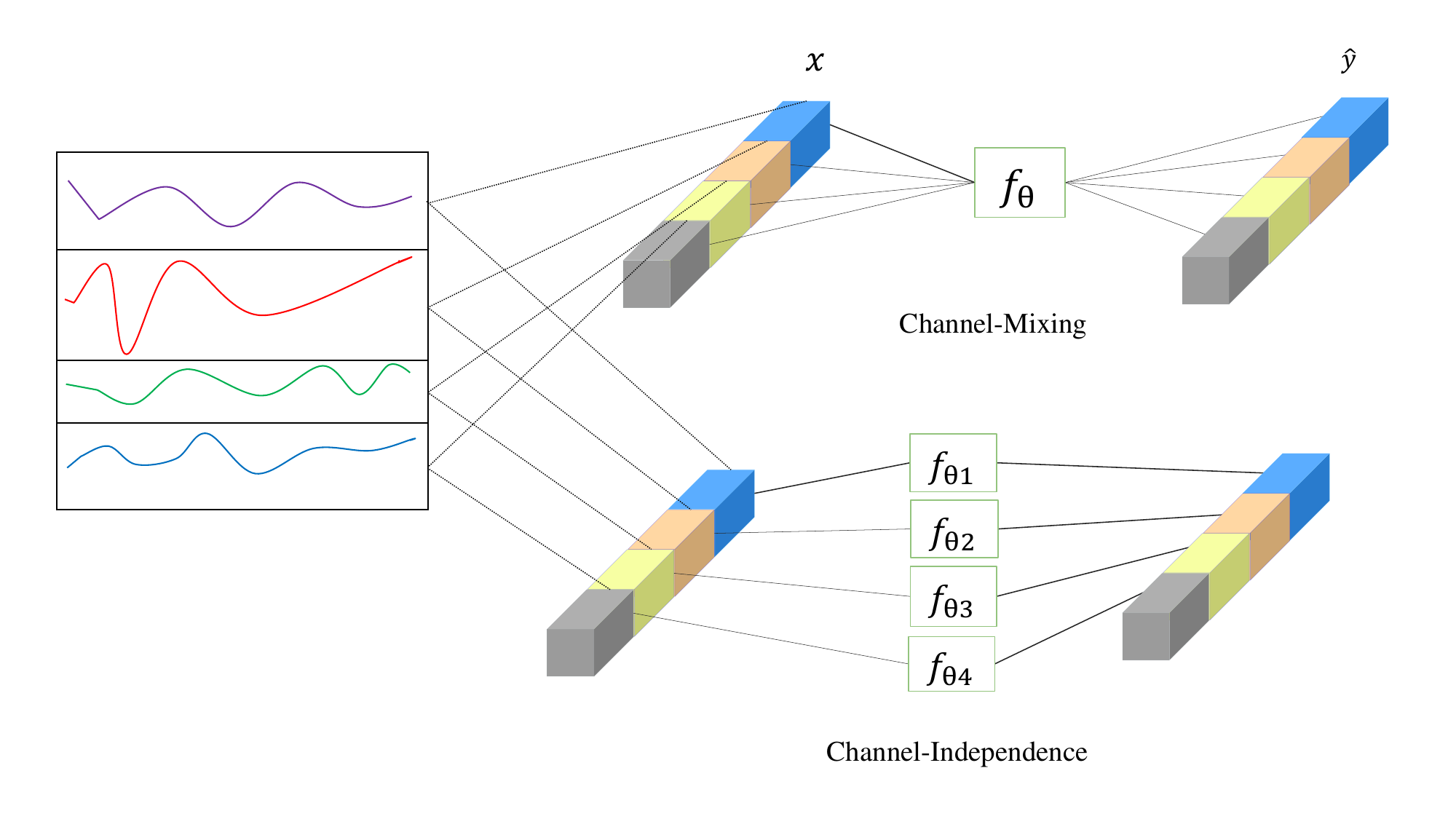}} %
\caption{Schematic illustration of the two mainstream channel designs.} 
\label{intro-fig1}
\end{center}
\end{figure}

The features of datasets such as ETTh1 and Electricity should be treated as independent samples rather than interdependent multivariate data. Given that the dependency relationships between input $x$ and output $y$ across different features are often similar, we find that one-channel methods may even surpass channel-independent methods in some cases when using DLinear \citep{16}. This makes us reconsider: Is treating time series features as channels necessary? Consequently, we propose Classification Auxiliary Channel-Independence (CACI), a new framework that supersedes the feature-as-channel paradigm. 

As depicted in Figure~\ref{frame}, CACI dynamically classifies series via a classifier, routing each series to respective predictors based on classification outcomes. Training a classifier requires category labels: we assign each series to the category whose corresponding predictor yields minimal prediction error, then utilize these labeled series to train the classifier. CCM \citep{chen24} computes cluster embeddings to cluster series and employs cross-attention to capture inter-channel dependencies, thereby realizing cluster-aware feedforward. Unlike CCM, whose clusters may misalign with prediction needs (e.g., grouping series requiring different predictors while separating those needing identical ones), our error-supervised design ensures prediction-relevant classification. This design not only enhances channel-prediction alignment but also reduces complexity from $\mathcal{O}(D)$ to $\mathcal{O}(1)$ with respect to $D$.

Seasonality-Trend decomposition enhances the accuracy of time series forecasting \citep{gardner1985, st90}. Autoformer employs moving averages for seasonal-trend decomposition \citep{23}, decomposing input sequences into trend component $t$ and seasonal component $s$. Specifically, given a known periodicity of 24, it derives $t$ via 24-step moving averaging and extracts $s$ by subtracting $t$ from the original series. DLinear separately applies linear mappings to both components and sums their predictions. However, this approach exhibits two limitations: (1) seasonal predictions utilize only temporal information while neglecting periodic emphasis; (2) temporal information within trend components remains entangled. To address these, we transform seasonal components into complex numbers, leveraging angular periodicity, perform complex-domain linear transformations, and then reconvert results to real seasonal predictions. For trend components, we decouple them into individual time-step states via the exponential smoothing method, apply linear mappings to these states, and finally recouple them into trend predictions. We designate the enhanced DLinear framework as TSLinear, which serves as the core predictor. Combined with the classifier, this integrated system constitutes CATS-Linear. The principal contributions of this work are:  

\begin{itemize}
\item We propose a novel channel design CACI, accompanied by a theoretical analysis of channel methods. CACI reduces parameter requirements of channel-independent approaches while boosting performance.
\item We refine DLinear's decomposition framework through enhanced seasonal-trend processing.
\item Experiments show that our hyperparameter-fixed model achieves SOTA comparable against hyperparameter-searching baselines and delivers an 8\% MSE reduction against unified hyperparameter baselines.
\end{itemize}

\section{RELATED WORK}

\textbf{Time Series Forecasting}. Classical time series forecasting models like ARIMA \citep{1968arima, box2015arima} leverage the stationarity of high-order differences, but exhibit limited capability in multi-step prediction. LSTM attempts to enhance RNN performance via gating mechanisms \citep{graves2012lstm, lai2018lstnet} since RNNs suffer from error accumulation in long-sequence forecasting. To address error accumulation in recurrent approaches, CNN-based models extract features and directly map them to the target space, with Temporal Convolutional Networks (TCN) specifically emphasizing the critical role of dilated convolutions in sequence forecasting \citep{borovykh2017conditional, franceschi2019unsupervised, luo2024moderntcn}.  In recent years, Transformer models \citep{li2019enhancing, liu2022pyraformer, wangcard} have gained prominence in time series forecasting. Concurrently, MLP methods have demonstrated competitive accuracy against established frameworks \citep{nbeats20, zhou22, liu2023koopa, ekambaram2023tsmixer, yi2024filternet}. Regarding emerging LLM-based approaches, some researchers question their forecasting efficacy \citep{tan2024language}. 

\textbf{Linear Model}. Distinct from the above models, linear models bypass feature extraction by adopting direct input-to-output mapping. Initially, LTSF-Linear demonstrates that linear models can surpass complex Transformer architectures, achieving state-of-the-art performance \citep{16}. Following this, a series of variants of the linear model were introduced \citep{li2024vlinear, wang2025clinear, genet2024temporal, ilbert2024analysing, rizvi2025bridging}. RLinear further validates the efficacy of linear models in forecasting periodic sequences \citep{li24}. FITS employs Fourier transforms to convert time series into the frequency domain, in which complex-valued linear mappings are subsequently applied \citep{xu24}. OLinear utilizes an adaptive orthogonal transformation matrix to encode and decode feature domains more efficiently, while introducing NormLin – a linear layer for replacing multi-head self-attention \citep{olinear2025}. Recent studies reveal that several existing linear variants are mathematically equivalent and collectively approximate linear regression \citep{Toner24}.

\section{METHODOLOGY}

\begin{figure*}[ht]\rmfamily      
\begin{center}
\centerline{\includegraphics[width=1.65\columnwidth]{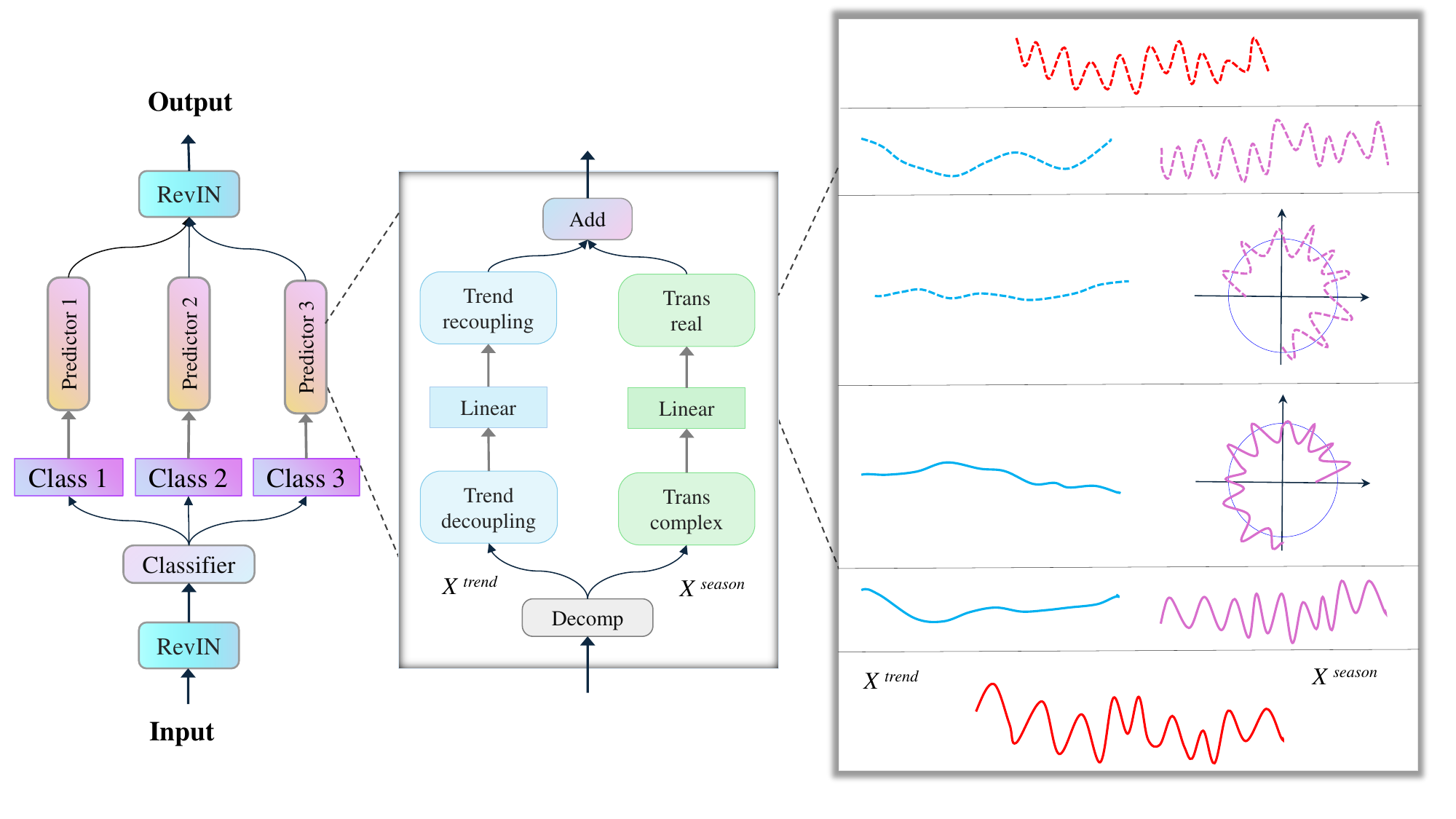}} %
\caption{Pipeline of CATS-Linear with TSLinear as predictors and RevIN as normalization method.} 
\label{frame}
\end{center}
\end{figure*}

Real-world time series exhibit significant distribution shifts. RevIN\citep{1} addresses shifts by first normalizing the input instance, eliminating its mean and variance information, and applying an affine transformation to ensure all input sequences have a mean of $\alpha \in R^D $ and a variance of $\beta \in R^D $. After prediction, the previously removed information is reintegrated into the final prediction. RevIN has been widely adopted as a module in advanced forecasting models \citep{27, liu24, li24}. For our model, RevIN additionally eliminates scale discrepancies between different features, thereby justifying the employment of a cross-channel forecasting setting.

\subsection{Classification Auxiliary Paradigm}
As illustrated in Figure \ref{frame}, CACI is a model-agnostic framework compatible with arbitrary forecasting models as predictors. CACI operates on the premise that different samples from the same time series may exhibit heterogeneous functional mappings from inputs to outputs. This heterogeneity diverges from channel-independence by not being determined by the dimension of the sequences. Thus, we classify sequences into distinct categories, training each category-specific predictor using samples in that category. For the Weather dataset, the classifier employs a four-layer network architecture. The initial three layers consist of 1D convolutional networks, each incorporating batch normalization and ReLU activation. The final layer processes features through a fully connected layer, followed by a softmax operation to output class probabilities. On other datasets, the classifier is a simpler two-layer MLP: the first layer applies tanh activation, while the second layer generates probability outputs via softmax. During training, the workflow follows Algorithm \ref{alg}. In the testing phase, to alleviate significant losses from misclassification, the final prediction is derived by weighting the outputs of individual predictors using the probability given by the classifier.

\subsection{Feature Decoupling Linear Model}
Periodicity is critical for accurate long-term time series forecasting, and linear models excel at capturing periodic patterns \citep{li24}. While existing linear approaches merely learn time dependency, we enhance periodicity information integration by converting seasonal components $s=(s_1, s_2,...s_p,...,s_L)$ into complex numbers as 
\begin{equation}
    z_p = s_p e^{jwp}, 1 \leq p\leq L,
\end{equation}
where $w$ is $2\pi / T$ and $j$ is the imaginary unit. Following this transformation, $z_p$ and $z_{p+T}$ are encoded as complex numbers sharing identical arguments with proximate moduli. After obtaining the complex-domain prediction $z^y$ via complex linear mapping, we convert it to the real domain, using:
\begin{equation}
    s^y_q = \delta (z^y_q) * \left| z^y_q \right|,  
\end{equation}
where
\begin{equation}
\delta (z^y_q) =  
\left\{  
     \begin{array}{lr}  
     1, \text{ if real part } Re(z^y_q / e^{jw(q+L)}) \geq 0, \\  
     -1,  \text{ otherwise }. \\  
     \end{array}  
\right.  
\end{equation}
The reason for adopting $\delta (z^y_q)$ as the sign stems from the transformation mechanism: positive $s_p$ values map to $z_p = s_p e^{jwp}$ with $Re(z_p/e^{jwp}) \geq 0$, while negative values map to $\left| s_p \right| e^{j(wp+\pi)}$  with $Re(z_p/e^{jwp}) \leq 0$. Consequently, during the reversion of $z_q$ to the real domain, we implement the inverse operation.
  
Unlike the seasonality in time series, which has a clear definition and analytical tools such as the Fourier transform, the trend of a time series lacks a universally accepted definition. However, there is a broad consensus that a trend means that the series value of the current time point is influenced by the past, resulting in a smoother time series. Exponential smoothing is widely applied to series with trend \citep{HOLT04, SMYLS2020, wo22}. Assuming the trend component derived from seasonality-trend decomposition $t$ can be decomposed into the sum of its current hidden state $h_i$ and its trend item $T_i$, i.e., $t_i$ = $h_i+T_i$. Further assuming that the influence of past state $h_i$ on subsequent observations decays exponentially, we obtain $t_i$ = $h_i+\alpha h_{i-1} +...+\alpha^{i-2} h_2+\alpha^{i-1} h_1$. Directly using the trend component to predict future trend components makes it difficult to guarantee smooth outputs. We perform a decoupling operation on the input trend component by $h_i = t_i - \alpha t_{i-1}$ to obtain the time point state, then apply linear mapping to derive the output time point state, and finally recoupling via convolution with a geometric sequence $[\alpha ^m, \alpha ^{m-1}, …, \alpha, 1]$. The justification analysis of the decoupling and recoupling operations is presented in Theorem \ref{theo3}.

\subsection{Training Schema}

\begin{algorithm}[tb]
\caption{Supervised Training Schema}
\label{alg}
\textbf{Input}: batch $(X, Y)$ = \{$(x,y)| x^{(b,d)}\in R^{L}$, $y^{(b,d)}\in R^{H}$, $1\leq b\leq B$, $1\leq d\leq D$\}; predictors $\{f_1, f_2,...,f_K\}$ \\
\textbf{Parameter}: Forecasting parameters $\theta$, classification parameters $\phi$\\
\begin{algorithmic}[1] 
\STATE $N_1 = N_2 = ...=N_K = B*D/K$
\STATE $X$ = RevIN.norm($X$) 
\STATE $\hat{C} = [\hat{C_1},\hat{C_2},...,\hat{C_K}] = Classifier(X)$ 
\STATE $X^c \leftarrow  X$, $Y^c\leftarrow Y$
\STATE $Y\leftarrow \emptyset$, $ \hat{Y}\leftarrow \emptyset$, $C\leftarrow \emptyset$
\FOR{$k\leftarrow 1$ to $K$}
\STATE $\tilde{Y}$ = RevIN.denorm($f_k(X^c)$)
\STATE $X_k = Top_{N_k}(-MSE(\tilde{Y}, Y^c))$ 
\STATE $Y_k$ = \{$y^{(b,d)}|$ if $x^{(b,d)}\in X_k$\}
\STATE $\hat{Y_k} = f_k(X_k)$ 
\STATE $C_k=[$1 if $x^{(b,d)}\in X_k$ else $0$ $]_{B\times D}$
\STATE $Y$.append($Y_k$), $\hat{Y}$.append($\hat{Y}_k$), $C$.append($C_k$)
\STATE $X^c = X^c - X_k$, $Y^c=Y^c-Y_k$
\ENDFOR
\STATE $\hat{Y}$ = RevIN.denorm($\hat{Y}$)
\STATE $\ell_f=MSE(Y,\hat{Y})$, $\ell_c=MSE(C,\hat{C})$
\STATE Update $\theta$ and $\phi$ using $\ell_f$ and $\ell_c$
\end{algorithmic}
\end{algorithm}

Algorithm \ref{alg} illustrates the training procedure of the Classification Auxiliary Channel-Independence for one batch. Given a batch size $B$ and feature dimension $D$, since we ignore feature discrepancies, we have $B * D$ instances. In line 2 RevIN is used for normalization, and in line 15 for denormalization. In line 3, the classifier outputs the probability of $x^{(b,d)} \in X_k$. That is, the element at the $b$-th row and $d$-th column of $\hat{C_k}$ represents $P(x^{(b,d)} \in X_k)$.

The most critical step of the training algorithm involves leveraging forecasting errors to assign categorical labels to input sequences for the classifier’s supervised training. For instance, upon receiving a batch, in line 7, we get prediction results of $f_1$. In line 8, we retrieve the top-$N_1$ instances with minimal MSE and designate them as the first category. In line 9, if instance $x^{(b,d)}$ is assigned to the first category, the element in the $b$-th row and $d$-th column of the matrix $C_1 \in R^{B\times D}$ is set to $1$; otherwise $0$. In line 13, we remove $X_1$ and $Y_1$ from the batch, and then we repeat the procedure for $k=2$. Crucially, the prediction $Y_k$ in line 10 is derived by processing $X_k$ through $f_k$ (i.e., $\hat{Y}_k = f_k(X_k)$), ensuring only utilizing $(X_k, Y_k)$ to train $f_k$ during backpropagation.

\section{Theoretical Analysis}

Previous research has demonstrated the equivalence between linear models and linear regression \citep{Toner24}. Therefore, we analyze linear models by analyzing linear regression. In this section, we assume $y \in R^{1}$ for simplicity. The fixed design linear regression assumes $y$ is a linear function of the input vector $x \in R^{L}$, but is disturbed by a random noise with $E[\epsilon]=0$ and $Var[\epsilon]=\sigma^2$:
\begin{equation}\label{meanstd}
y = x^{T}\theta^{*} + \epsilon.
\end{equation}
In linear regression, the training samples are concatenated and written into design matrices $Y \in R^{N}$, $X \in R^{L \times N}$, $\varepsilon \in R^{N}$, and $\Psi = X^TX \in R^{L \times L}$. The Ordinary Least Squares (OLS) has given that the unbiased estimator of $\theta^{*}$ is $(X^TX)^{-1}X^TY$. 

\begin{defi}[Expected Risk]
Given $\theta$, which determines a function $f: \mathcal{X} \rightarrow \mathcal{Y}$, a loss function $l: \mathcal{Y \times Y} \rightarrow R$, the expected risk of $\theta$ is defined as, 
\begin{equation}
    \mathcal{R}(\theta) = E[l(y,f(x))] = \int_{\mathcal{X \times Y}} l(y,f(x)) dp(x,y).
\end{equation}
\end{defi}
The minimum expected risk is the Bayes risk $\mathcal{R^*}$ \citep{bach}. We call the difference between $\mathcal{R}(\theta)$ and $\mathcal{R^*}$ the excess risk of $\theta$.

\begin{lemma}[Risk Decomposition]
\label{lemma1}
We have $\mathcal{R^*} = \sigma^2$ and $\mathcal{R}(\theta) -\mathcal{R^*} = \Vert \theta - \theta^* \Vert^2_{\Psi}$ for any $\theta \in \Theta$, where $\Vert \theta \Vert^2_{\Psi} = \frac{1}{N}\theta^T\Psi\theta$ is a Mahalanobis distance norm. Particularly, if $\hat{\theta}$ is a random variable such as an estimator of $\theta^*$, then    
\begin{equation}
    E[\mathcal{R}(\hat{\theta})] -\mathcal{R^*} =  \underbrace{\Vert E[\hat{\theta}] - \theta^* \Vert^2_{\Psi}}_{bias\; part} + \underbrace{E[\Vert \hat{\theta} - E[\hat{\theta}] \Vert^2_{\Psi}]}_{variance\; part }.
\end{equation}
\end{lemma}

This lemma is Proposition 3.3 in \citep{bach}. Now we can discuss the situation of multiple linear models. Suppose the total samples are from $K$ classes, each with $N_1$, $N_2$, ..., $N_K$ ($N = N_1 + N_2 + ...+ N_K$) samples in its class. We use $\theta_k$ and $X_k$ to denote the parameters and design matrix of the $k$-th class($Y_k$, $\varepsilon_k$, and $\Psi_k$ respectively). We have the following theorems.     

\begin{theo}[]
\label{theo1}
If estimators $\hat{\theta}_1$, $\hat{\theta}_2$, ..., $\hat{\theta}_K$ are the OLS estimators $(X_k^TX_k)^{-1}X_k^TY_k$ computed using their own $X_k$ and $Y_k$:   
\begin{equation}
    E[\mathcal{R}(\hat{\theta}_1, \hat{\theta}_2, ..., \hat{\theta}_K)] -\mathcal{R^*} = \underbrace{0}_{bias\; part} + \underbrace{\frac{KL}{N}\sigma^2}_{variance\; part}.
\end{equation}
\end{theo}

\begin{proof}
    $E[\mathcal{R}(\hat{\theta}_1, ..., \hat{\theta}_K)] -\mathcal{R^*}$=$\sum_{1}^K \frac{N_k}{N} (E[\mathcal{R}(\hat{\theta}_k)] -\mathcal{R}_{k}^{*})$. Using Proposition 3.5 in \citep{bach} which concludes $E[\mathcal{R}(\hat{\theta}_k)] -\mathcal{R}_{k}^{*} = \frac{L}{N_k}\sigma^2$. Now $E[\mathcal{R}(\hat{\theta}_1, ..., \hat{\theta}_K)]-\mathcal{R^*}$ = $\sum_{1}^K \frac{N_k}{N} * (E[\mathcal{R}(\hat{\theta}_k)] -\mathcal{R}_{k}^{*})$ = $\sum_{1}^K \frac{N_k}{N} * \frac{L}{N_k}\sigma^2$ = $\frac{KL}{N}\sigma^2$.   
\end{proof}

\begin{theo}[]
\label{theo2}
If the classes of the samples are unknown and one learns a global linear model, then
\begin{equation}
E[\mathcal{R}(\hat{\theta}, \hat{\theta}, ..., \hat{\theta})] -\mathcal{R^*} = \underbrace{\sum_{k=1}^K \frac{N_k}{N}\Vert \overline{\theta} - \theta^*_k \Vert^2_{\Psi_k}}_{bias\; part} \; \: + \underbrace{\frac{L}{N}\sigma^2}_{variance\;part}
\end{equation}
 where $\overline{\theta} = (\sum_1^K \Psi_k)^{-1}(\sum_1^k\Psi_k\theta_k^*)$.
\end{theo} 

\begin{proof}
For the bias part, we only need to prove $E[\hat{\theta}]=\overline{\theta}$. Noting that $E[Y_k]$ = $E[X_k\theta^*_k+\varepsilon_k]$ = $X_k\theta^*_k$, we have $E[\hat{\theta}]$ = $E[(X^TX)^{-1}X^TY]$ = $E[(\sum_1^K \Psi_k)^{-1}(\sum_1^K X_k^TY_k)]$ = $(\sum_1^K \Psi_k)^{-1}(\sum_1^K X_k^TE[Y_k])$ = $(\sum_1^K \Psi_k)^{-1}$ $(\sum_1^K X_k^TX_k$ $\theta^*_k)$ = $(\sum_1^K \Psi_k)^{-1}(\sum_1^K \Psi_k\theta^*_k)$.

For the variance part, we use a conclusion for matrix multiplication. That is, given a column vector $a$ and a symmetric matrix $A$, $a^TAa$ equals the trace of $Aaa^T$. $\sum_{k=1}^K \frac{N_k}{N}E[\Vert \hat{\theta_k} - E[\hat{\theta_k}] \Vert^2_{\Psi_k}]$ \\ = $\sum_{k=1}^K \frac{N_k}{N} E[(\sum_{1}^KX_k^T\varepsilon_k)^T\Psi^{-1} (\frac{1}{N_k}\Psi_k) \Psi^{-1}(\sum_{1}^KX_k^T\varepsilon_k)]$ \\ 
= $E[(\sum_{1}^KX_k^T\varepsilon_k)^T\Psi^{-1} (\sum_{k=1}^K \frac{N_k}{N}\frac{1}{N_k}\Psi_k) \Psi^{-1}(\sum_{1}^KX_k^T\varepsilon_k)]$ \\
= $\frac{1}{N}E[(\sum_{1}^KX_k^T\varepsilon_k)^T \Psi^{-1} (\sum_{1}^KX_k^T\varepsilon_k)]$ \\
= $\frac{1}{N}\sum_{k=1}^KE[(X_k^T\varepsilon_k)^T \Psi^{-1} (X_k^T\varepsilon_k)]$ \\
= $\frac{1}{N}\sum_{k=1}^KE[(\varepsilon_k^TX_k) \Psi^{-1} (X_k^T\varepsilon_k)]$ \\
= $\frac{1}{N}\sum_{k=1}^KE[tr(X_k \Psi^{-1} X_k^T\varepsilon_k\varepsilon_k^T)]$ ... $a^TAa$ $=$ $tr(Aaa^T)$ \\ 
= $\frac{1}{N}\sum_{k=1}^Ktr(X_k \Psi^{-1} X_k^TE[\varepsilon_k\varepsilon_k^T])$ \\
= $\frac{1}{N}\sum_{k=1}^Ktr(X_k \Psi^{-1} X_k^T\sigma^2I_{N_k})$ \\
= $\frac{1}{N}\sigma^2\sum_{k=1}^Ktr(X_k \Psi^{-1} X_k^T)$ \\
= $\frac{1}{N}\sigma^2\sum_{k=1}^Ktr(\Psi^{-1} X_k^TX_k)$ ... $tr(AB) = tr(BA)$\\
= $\frac{1}{N}\sigma^2tr(\Psi^{-1} \sum_{k=1}^K(X_k^TX_k))$ \\
= $\frac{1}{N}\sigma^2tr(I_{L})$ \\
= $\frac{L}{N}\sigma^2$
\end{proof}

We now analyze the expected excess risk for each channel design. For the one-channel method, Theorem 2 characterizes its expected excess risk. Compared to CACI, its variance error is merely $1/K$ of the former's. This reduction stems from training a single model with substantial samples, diminishing stochasticity-induced errors. However, the one-channel method incurs a bias error. When training samples exhibit heterogeneous functional mappings from $x$ to $y$, the estimator $\hat{\theta}$ becomes biased. 


The channel-mixing method yields an unbiased estimator and achieves the smallest bias error among all approaches. However, it suffers from the largest variance error. Assuming the data has $D$ features, since its training sample size is reduced to $1/D$ of the one-channel method while its parameter count increases by a factor of $D$, its variance error becomes $D^2$ times that of the one-channel method and $D^2/K$ times that of CACI. Prior research demonstrates that channel-mixing leads to lower capacity compared to channel-independent alternatives \citep{27}, as well as lower robustness. This occurs because for multivariate time series, $y[i]$ can be effectively explained by its own historical values $x[i]$. Incorporating other features $x[j]$ ($j \neq i$) provides limited bias reduction while substantially amplifying variance.

Channel-independence design can be viewed as a special case of CACI where instances are categorized into $D$ classes based on their feature dimension. Consequently, its variance error is $D$ times that of the one-channel method and $D/K$ times that of CACI. Notably, this method incurs a bias error when instances from the same feature follow heterogeneous functional mappings, or instances from different features may share identical functional mappings.

Regarding CACI, our error-supervised approach minimizes bias error by assigning labels using posterior prediction errors. However, during testing, misclassification by the classifier may still introduce a bias error. Additionally, CACI's variance error is $K$ times that of the one-channel method.

\begin{theo}[Trend Decoupling]
\label{theo3}
If the time series can be expressed as
$\left\{  
     \begin{array}{lr}  
     t_1=h_1, &  \\  
     t_2=\alpha h_1 + h_2, & \\  
     t_3=\alpha^2 h_1 +\alpha h_2+ h_3, & \\   
     \; \dots
     \end{array}  
\right.$ then 
$\left\{  
     \begin{array}{lr}  
     h_1=t_1, &  \\  
     h_2=t_2 - \alpha t_1, & \\  
     h_3=t_3 - \alpha t_2, & \\   
     \; \dots
     \end{array}  
\right. $
\end{theo}

\begin{proof}
    $h_i$ = $t_i-(\alpha h_{i-1} +...+\alpha^{i-2} h_2+\alpha^{i-1} h_1)$ = $t_i-\alpha(h_{i-1} +...+\alpha^{i-3} h_2+\alpha^{i-2} h_1)$ = $t_i-\alpha t_{i-1}$.   
\end{proof}

Theorem \ref{theo3} justifies the decoupling operation. We now explain the rationale for the recoupling operation. Given that $\sum_{i=m+1}^{\infty} \alpha^i= \alpha^{m+1} / (1-\alpha)$ is far smaller than $\sum_{i=0}^{m} \alpha^i = (1-\alpha^{m+1}) /(1-\alpha)$, and considering the $h_i$ terms as a bounded sequence, calculating $t_i$ for $i\geq m+1$ using $t_i$ = $h_i+\alpha h_{i-1} +\dots+\alpha^{i-2} h_2+\alpha^{i-1} h_1$, requires only convolving $[h_1, h_2, h_3,\dots]$ with kernel $[\alpha ^m, \alpha ^{m-1}, \dots, \alpha, 1]$.

\section{EXPERIMENTS}

\begin{table*}[htb]\rmfamily
    \caption{Forecasting results with target lengths $H \in \{96, 192, 336, 720\}$. The input length $L$ is 336 for CATS-Linear, PatchTST, and DLinear, 720 for TiDE, grid-searched $L$ in \{90, 180, 360, 720\} for FITS, and 96 for all others. The best results are highlighted in bold, while the second-best results are underlined. The bottom row shows the count of the best results for each column.}
    \label{main-tab1}
    \setlength{\tabcolsep}{0.9mm}
    \begin{small}
    \centering
    \begin{tabular}{c|c|cc|cccccc|cccccccc|c}
        \toprule
        \multirow{3}{*}{\rotatebox{90}{Data}}
         & Method & \multicolumn{2}{c|}{CACI} & \multicolumn{6}{c|}{Channel-Independence} & \multicolumn{8}{c|}{Channel-Mixing} & OC \\
        & Model & \multicolumn{2}{c|}{CATS-Linear} & \multicolumn{2}{c}{DLinear} & \multicolumn{2}{c}{PatchTST} & \multicolumn{2}{c|}{TiDE}& \multicolumn{2}{c}{OLinear} & \multicolumn{2}{c}{TimeMixer++} & \multicolumn{2}{c}{iTransformer}   & \multicolumn{2}{c|}{TimesNet} & FITS\\
        & Metric & MSE & MAE & MSE & MAE & MSE & MAE & MSE & MAE & MSE & MAE & MSE & MAE & MSE & MAE & MSE & MAE & MSE \\
        
        \midrule
        \multirow{4}{*}{\rotatebox{90}{Weather}} 
        & 96   & \textbf{0.125} & 0.216 & 0.176 & 0.237 & 0.149 & \underline{0.198} & 0.166 & 0.222 & 0.153 & \textbf{0.190} & 0.155 & 0.205 & 0.174 & 0.214  & 0.172 & 0.220 & \underline{0.143} \\
        & 192  & \textbf{0.183} & 0.268 & 0.220 & 0.282 & \underline{0.194} & 0.241 & 0.209 & 0.263 & 0.200 & \textbf{0.235} & 0.201 & 0.245 & 0.221 & 0.254   & 0.219 & 0.261 & \underline{0.186} \\
        & 336  & 0.244 & 0.315 & 0.265 & 0.319 & 0.245 & \underline{0.282} & 0.254 & 0.301 & 0.258 & 0.280 & \underline{0.237} & \textbf{0.265} & 0.278 & 0.296  & 0.280 & 0.306 & \textbf{0.236} \\
        & 720  & 0.344 & 0.389 & 0.323 & 0.362 & 0.314 & \underline{0.334} & 0.313 & 0.340 & 0.337 & \textbf{0.333} & \underline{0.312} & \underline{0.334} & 0.358 & 0.347  & 0.365 & 0.359 & \textbf{0.307} \\

        \midrule
        \multirow{4}{*}{\rotatebox{90}{Electricity}} 
        & 96   & 0.140 & 0.234 & 0.140 & 0.237 & \textbf{0.129} & \underline{0.222} & 0.132 & 0.229 & \underline{0.131} & \textbf{0.221} & 0.135 & \underline{0.222} & 0.148 & 0.240  & 0.168 & 0.272 & 0.134 \\
        & 192  & 0.153 & 0.247 & 0.153 & 0.249  & \textbf{0.147} & 0.240 & \textbf{0.147} & 0.243 & 0.150 & \underline{0.238} & \textbf{0.147} & \textbf{0.235} & 0.162 & 0.253  & 0.184 & 0.289 & \underline{0.149} \\
        & 336  & 0.168 & 0.262 & 0.169 & 0.267 & \underline{0.163} & 0.259 & \textbf{0.161} & 0.261 & 0.165 & \underline{0.254} & 0.164 & \textbf{0.245} & 0.178 & 0.269 & 0.198 & 0.300 & 0.165 \\
        & 720  & 0.208 & 0.294 & 0.203 & 0.301 & 0.197 & \underline{0.290} & \underline{0.196} & 0.294 & \textbf{0.191} & \textbf{0.279} & 0.212 & 0.310 & 0.225 & 0.317  & 0.220 & 0.320 & 0.203 \\

        \midrule
        \multirow{4}{*}{\rotatebox{90}{Traffic}} 
        & 96   & 0.416 & 0.281 & 0.410 & 0.282 & \underline{0.360} & \underline{0.249} & \textbf{0.336} & 0.253 & 0.398 & \textbf{0.226} & 0.392 & 0.253 & 0.395  & 0.268 & 0.593 & 0.321 & 0.385 \\
        & 192  & 0.430 & 0.288 & 0.423 & 0.287 & \underline{0.379} & \underline{0.256} & \textbf{0.346} & 0.257 & 0.439 & \textbf{0.241} & 0.402 & 0.258 & 0.417 & 0.276  & 0.617 & 0.336 & 0.397 \\
        & 336  & 0.442 & 0.293 & 0.436 & 0.296 & \underline{0.392} & 0.264 & \textbf{0.355} & \underline{0.260} & 0.464 & \textbf{0.250} & 0.428 & 0.263 & 0.433 & 0.283 & 0.629 & 0.336 & 0.410 \\
        & 720  & 0.466 & 0.311 & 0.466 & 0.315 & \underline{0.432} & 0.286 & \textbf{0.386} & \underline{0.273} & 0.502 & \textbf{0.270} & 0.441 & 0.282 & 0.467 & 0.302  & 0.640 & 0.350 & 0.448 \\
        
        \midrule
        \multirow{4}{*}{\rotatebox{90}{ETTh1}} 
        & 96   & \textbf{0.360} & \textbf{0.395} & 0.375 & 0.399 & 0.370 & \underline{0.400} & 0.375 & 0.398 & \textbf{0.360} & 0.382 & \underline{0.361} & 0.403 & 0.386 & 0.405  & 0.384 & 0.402 & 0.372\\
        & 192  & \textbf{0.404} & \textbf{0.413} & \underline{0.405} & 0.416 & 0.413 & 0.429 & 0.412 & 0.422 & 0.416 & \underline{0.414} & 0.416 & 0.441 & 0.441 & 0.436  & 0.436 & 0.429 & \textbf{0.404}\\
        & 336  & 0.430 & \textbf{0.432} & 0.439 & 0.443 & \textbf{0.422} & 0.440 & 0.435 & \underline{0.433} & 0.457 & 0.438 & \underline{0.430} & 0.434 & 0.487 & 0.458 & 0.491 & 0.469 & \underline{0.427}\\
        & 720  & \underline{0.440} & \textbf{0.450} & 0.472 & 0.490 & 0.447 & 0.468 & 0.454 & 0.465 & 0.463 & 0.462 & 0.467 & \underline{0.451} & 0.503 & 0.491  & 0.521 & 0.500 & \textbf{0.424}\\
        
        \midrule
        \multirow{4}{*}{\rotatebox{90}{ETTh2}} 
        & 96   & \textbf{0.269} & \textbf{0.326} & 0.289 & 0.353 & \underline{0.274} & 0.337 & 0.270 & 0.336 & 0.284 & 0.329 & 0.276 & \underline{0.328} & 0.297 & 0.349  & 0.340 & 0.374 & 0.271\\
        & 192  & 0.335 & \textbf{0.373} & 0.383 & 0.418 & \textbf{0.314} & \textbf{0.382} & 0.332 & 0.380 & 0.360 & \underline{0.379} & 0.342 & \underline{0.379} & 0.380 & 0.400  & 0.402 & 0.414 & \underline{0.331} \\
        & 336  & 0.355 & \underline{0.395} & 0.448 & 0.465 & \textbf{0.329} & \textbf{0.384} & 0.360 & 0.407 & 0.409 & 0.415 & \underline{0.346} & 0.398 & 0.428 & 0.432  & 0.452 & 0.452 & 0.354\\
        & 720  & 0.398 & \underline{0.429} & 0.605 & 0.551 & \underline{0.379} & \textbf{0.422} & 0.419 & 0.451 & 0.415 & 0.431 & 0.392 & 0.415 & 0.427 & 0.445  & 0.462 & 0.468 & \textbf{0.377}\\

        \midrule
        \multirow{4}{*}{\rotatebox{90}{ETTm1}} 
        & 96   & \textbf{0.287} & \textbf{0.333} & 0.299 & 0.343 & \underline{0.293} & 0.346 & 0.306 & 0.349 & 0.302 & \underline{0.334} & 0.310 & \underline{0.334} & 0.334 & 0.368  & 0.338 & 0.375 & 0.303 \\
        & 192  & \textbf{0.328} & \textbf{0.360} & 0.335 & \underline{0.365} & \underline{0.333} & 0.370 & 0.335 & 0.366 & 0.357 & 0.363 & 0.348 & 0.362 & 0.377 & 0.391  & 0.374 & 0.387 & 0.337 \\
        & 336  & \textbf{0.364} & \textbf{0.381} & \underline{0.369} & 0.386 & 0.369 & 0.392 & \textbf{0.364} & \underline{0.384} & 0.387 & 0.385 & 0.376 & 0.391 & 0.426 & 0.420  & 0.410 & 0.411 & \underline{0.366} \\
        & 720  & 0.424 & \underline{0.416} & 0.425 & 0.421 & 0.416 & 0.420 & \textbf{0.413} & \textbf{0.413} & 0.452 & 0.426 & 0.440 & 0.423 & 0.491 & 0.459  & 0.478 & 0.450 & \underline{0.415} \\

         \midrule
        \multirow{4}{*}{\rotatebox{90}{ETTm2}} 
        & 96   & \textbf{0.160} & \underline{0.246} & 0.167 & 0.260 & 0.166 & \underline{0.256} & \underline{0.161} & 0.251 & 0.169 & 0.249 & 0.170 & \textbf{0.245} & 0.180 & 0.264  & 0.187 & 0.267 & 0.162 \\
        & 192  & \textbf{0.213} & \textbf{0.283} & 0.224 & 0.303 & 0.223 & 0.296 & \underline{0.215} & \underline{0.289} & 0.232 & 0.290 & 0.229 & 0.291 & 0.250 & 0.309  & 0.249 & 0.309 & 0.216 \\
        & 336  & \textbf{0.265} & 0.319 & 0.281 & 0.342 & 0.274 & 0.329 & \underline{0.267} & \underline{0.326} & 0.291 & 0.328 & 0.303 & 0.343 & 0.311 & 0.348  & 0.321 & 0.351 & 0.268 \\
        & 720  & \underline{0.350} & \textbf{0.373} & 0.397 & 0.421 & 0.362 & 0.385 & 0.352 & \underline{0.383} & 0.389 & 0.387 & 0.373 & 0.399 & 0.412 & 0.407  & 0.408 & 0.403 & \textbf{0.348} \\
        \midrule
        \multicolumn{2}{c|}{Count} & 11 & 11 & & & 5 & 3 & 8 & 1 & 2 & 9  & 1 & 4 & & & & & 6 \\
        \bottomrule
    \end{tabular}
    \end{small}
\end{table*}

\begin{table*}[htb]\rmfamily
    \caption{Comparison with the unified hyperparameter baselines. The results are averaged from the four forecasting lengths. The input length $L$ is 336 for CATS-Linear and 96 for the rest.}
    \label{main-tab2}
    \setlength{\tabcolsep}{1mm}
    \centering
    \begin{tabular}{c|cccccccccccccc|c}
        \toprule
        Model & \multicolumn{2}{c}{CATS-Linear} & \multicolumn{2}{c}{TimeMixer}  & \multicolumn{2}{c}{FiLM}  & \multicolumn{2}{c}{MICN} & \multicolumn{2}{c}{ Crossformer} & \multicolumn{2}{c}{Autoformer}  & \multicolumn{2}{c|}{TimesNet} & \\
        Metric & MSE & MAE & MSE & MAE & MSE & MAE & MSE & MAE & MSE & MAE & MSE & MAE & MSE & MAE & Imp.  \\
        
        \midrule
        Weather & \textbf{0.224} & \underline{0.297} & \underline{0.240} & \textbf{0.271} & 0.271 & 0.291 & 0.268 & 0.321 & 0.264  & 0.320 & 0.338 & 0.382 & 0.251 & 0.294 & 6.67\% \\
        
        Electricity  & \textbf{0.167} & \textbf{0.259} & \underline{0.182} & \underline{0.272} & 0.223 & 0.302 & 0.196 & 0.309 & 0.244  & 0.334 & 0.227 & 0.338 & 0.193 & 0.304 & 8.24\% \\

        Traffic  & \textbf{0.439} & \textbf{0.293} & \underline{0.484} & \underline{0.297} & 0.637 & 0.384 & 0.593 & 0.356 & 0.667  & 0.426 & 0.628 & 0.379 & 0.620 & 0.336 & 9.30\% \\
        
        ETTh1  & \textbf{0.409} & \textbf{0.423} & \underline{0.447} & \underline{0.440} & 0.516 & 0.483 & 0.475 & 0.480 & 0.529 & 0.522 & 0.496 & 0.487 & 0.495 & 0.450 & 8.50\% \\
        
        ETTh2  & \textbf{0.339} & \textbf{0.381} & \underline{0.364} & \underline{0.395} & 0.402 & 0.420 & 0.574 & 0.531 & 0.942  & 0.684 & 0.450 & 0.459 & 0.414 & 0.427 & 6.87\% \\

        ETTm1  & \textbf{0.351} & \textbf{0.373} & \underline{0.381} & \underline{0.395} & 0.411 & 0.402 & 0.423 & 0.422 & 0.513  & 0.495 & 0.588 & 0.517 & 0.400 & 0.406 & 7.87\% \\

        ETTm2  & \textbf{0.247} & \textbf{0.305} & \underline{0.275} & \underline{0.323} & 0.287 & 0.329 & 0.353 & 0.402 & 0.757  & 0.610 & 0.327 & 0.371 & 0.291 & 0.333 & 10.2\% \\
        
        \bottomrule
    \end{tabular}
\end{table*}

\textbf{Benchmarks and Baselines}. We conduct experiments on seven public benchmark datasets for long-term forecasting, partitioning the four ETT datasets into training/validation/test sets at 6:2:2 ratios while applying 7:1:2 splits to the other three datasets. Our evaluation encompasses linear models DLinear \citep{16}, FITS \citep{xu24} and OLinear \citep{olinear2025}, Transformer models PatchTST \citep{27}, iTransformer \citep{liu24} and Autoformer \citep{23}, MLP architectures TiDE \citep{2023tide} and TimeMixer++ \citep{25timemixer2}, and temporal convolutional networks TimesNet \citep{23timesnet}, with uniform forecasting horizons of {96, 192, 336, 720}. Results in Table \ref{main-tab1} are directly collected from their original papers. The lookback window is 336 for PatchTST and DLinear, 720 for TiDE, and 96 for all other models. We use Mean Squared Error (MSE) and Mean Absolute Error (MAE) as metrics.

\textbf{Hyperparameters for CATS-Linear}. All experiments for CATS-Linear are conducted with fixed parameter configurations on an NVIDIA RTX 4060Ti 16GB GPU. We utilize the Adam optimizer \citep{kingma2014adam} with a fixed learning rate of 1e-4 for predictors and 1e-5 for classifiers across all datasets. We use a CNN as the Classifier for Weather and an MLP for other datasets, with the number of classification categories uniformly set to 10 equally sized groups. Increasing dimensionality while maintaining fixed category counts is equivalent to increasing sample size for CATS-Linear. Consequently, batch size is set to 128 for low-dimensional datasets, $32$ for Electricity, and $8$ for Traffic. For linear mapping parameters, $\alpha$ remains fixed at $0.5$, $m$ at $10$, while periodicity $T$ is configured as $144$ for Weather, $96$ for ETTm, and $24$ for all others. 

\subsection{Main Results}

The main results are averages over three independent runs, shown in Table \ref{main-tab1}. CATS-Linear achieves 24 top-ranked results, securing first place among all models. OLinear ranks second with 11 best outcomes. Crucially, \textbf{CATS-Linear's results are obtained with unified hyperparameters in all cases}, whereas competing models reflect the best results of several hyperparameter settings, making this achievement particularly significant. Compared to DLinear, CATS-Linear not only reduces MSE by approximately 10\% but also delivers superior stability. CATS-Linear exhibits moderate performance on the Electricity and Traffic datasets, which we conjecture may stem from stronger adherence to consistent linear functional mappings across samples within these domains.

To benchmark fixed-parameter performance, we compare CATS-Linear against 6 models, including Crossformer \citep{23crossformer} and MICN \citep{23micn}. Experimental results for comparative models are reported from the TimeMixer \citep{23timemixer}. Table \ref{main-tab2} presents results averaged across forecasting horizons \{96, 192, 336, 720\}, with a 336-step lookback window for CATS-Linear and a 96-step for competitors. Notably, CATS-Linear achieves the lowest error in 13 out of 14 experimental settings, demonstrating exceptional stability. Compared to runner-up TimeMixer, CATS-Linear reduces MSE by 8\%. Against third-ranked TimesNet, it achieves MSE reductions of 10.76\%, 13.47\%, 29.19\%, 17.37\%, 18.12\%, 12.25\%, and 15.12\% across the seven datasets.

To validate the efficiency of CATS-Linear, we compare it against models such as FEDformer, with Linear \citep{16} as the predictor. The results demonstrate that CATS-Linear achieves the lowest values in the number of parameters, Multiply-Accumulate Operations (MACs), and inference time, as shown in \ref{Ablation-tab4}. Furthermore, employing CACI leads to a significant reduction in both the parameter count and inference time for DLinear. 

\begin{table}[htb]\rmfamily
    \caption{Parameter numbers, MACs, and inference time with L=96 and H=720 on Electricity. The inference time is derived by fixing the batch size at 32.}
    \label{Ablation-tab4}
    \setlength{\tabcolsep}{1mm}
    \begin{small}
    \centering
    \begin{tabular}{c|ccc}
        \toprule
        Model & Parameter & MAC & Infer.-Time \\
        
        \midrule
        TimesNet & 301.7M & 1226.49G & N/A \\
        Autoformer & 14.91M & 4.41G &  213.77ms \\
        FEDformer & 20.68M & 4.41G  & 74.17ms \\
        FiLM  &14.91M & 5.97G &  184.45ms \\
        DLinear (CI) & 44.38M & 89.09M & 73.67ms \\
        DLinear (CACI) & 1.41M & 912.85M &  29.49ms \\
        CATS-Linear  & 0.72M & 463.40M & 22.14ms \\
    
        \bottomrule
    \end{tabular}
    \end{small}
\end{table}

\begin{table}[htb]\rmfamily
    \caption{Ablation study of the modules of CATS-Linear. CATS-Linear without CACI employs one channel setting.}
    \label{Ablation-tab3}
    \setlength{\tabcolsep}{0.78mm}
    \begin{small}
    \centering
    \begin{tabular}{c|c|cccccccc}
        \toprule
        \multirow{2}{*}{\rotatebox{90}{Data}}
        &  & \multicolumn{2}{c}{CATS-Linear} & \multicolumn{2}{c}{w/o RevIN} & \multicolumn{2}{c}{w/o TSLinear} & \multicolumn{2}{c}{w/o CACI} \\
        & & MSE & MAE & MSE & MAE & MSE & MAE & MSE & MAE \\
        
        \midrule
        \multirow{4}{*}{\rotatebox{90}{Weather}} 
        & 96   & 0.125 & 0.216 & 0.126 & \textbf{0.215} & 0.130 & 0.221 & 0.140 & 0.233 \\
        & 192  & \textbf{0.183} & \textbf{0.268} & \textbf{0.183} & 0.269 & 0.184 & 0.273 & 0.194 & 0.281 \\
        & 336  & 0.244 & \textbf{0.315} & 0.248 & 0.319 & \textbf{0.243} & 0.318 & 0.254 & 0.326 \\
        & 720  & \textbf{0.344} & 0.389 & 0.345 & \textbf{0.388} & 0.348 & 0.392 & \textbf{0.344} & 0.389 \\
        
        \midrule
        \multirow{4}{*}{\rotatebox{90}{Electricity}} 
         & 96   & \textbf{0.140} & \textbf{0.234} & 0.141 & 0.235 & 0.141 & 0.236 & 0.145 & 0.241 \\
        & 192  & \textbf{0.153} & \textbf{0.247} & 0.154 & 0.248 & 0.154 & 0.248 & 0.156 & 0.250 \\
        & 336  & 0.168 & 0.262 & 0.169 & \textbf{0.261} & \textbf{0.167} & 0.262 & 0.170 & 0.265 \\
        & 720  & \textbf{0.208} & \textbf{0.294} & 0.209 & 0.296 & 0.209 & 0.297 & 0.211 & 0.298 \\

        \midrule
        \multirow{4}{*}{\rotatebox{90}{ETTh1}} 
         & 96   & \textbf{0.360} & 0.395 & 0.368 & \textbf{0.392} & 0.378 & 0.400 & 0.376 & 0.398 \\
        & 192  & \textbf{0.404} & \textbf{0.413} & 0.406 & 0.415 & 0.414 & 0.421 & 0.415 & 0.423 \\
        & 336  & 0.430 & 0.432 & \textbf{0.429} & \textbf{0.430} & 0.433 & 0.434 & 0.437 & 0.436 \\
        & 720  & 0.440 & \textbf{0.450} & \textbf{0.430} & 0.453 & 0.450 & 0.465 & 0.457 & 0.469 \\
    
        \bottomrule
    \end{tabular}
    \end{small}
\end{table}

\subsection{Ablation Study}

\begin{figure*}[tb]\rmfamily
    \centering
    \begin{minipage}[t]{0.16\linewidth}
        \centering
        \includegraphics[width=1\linewidth]{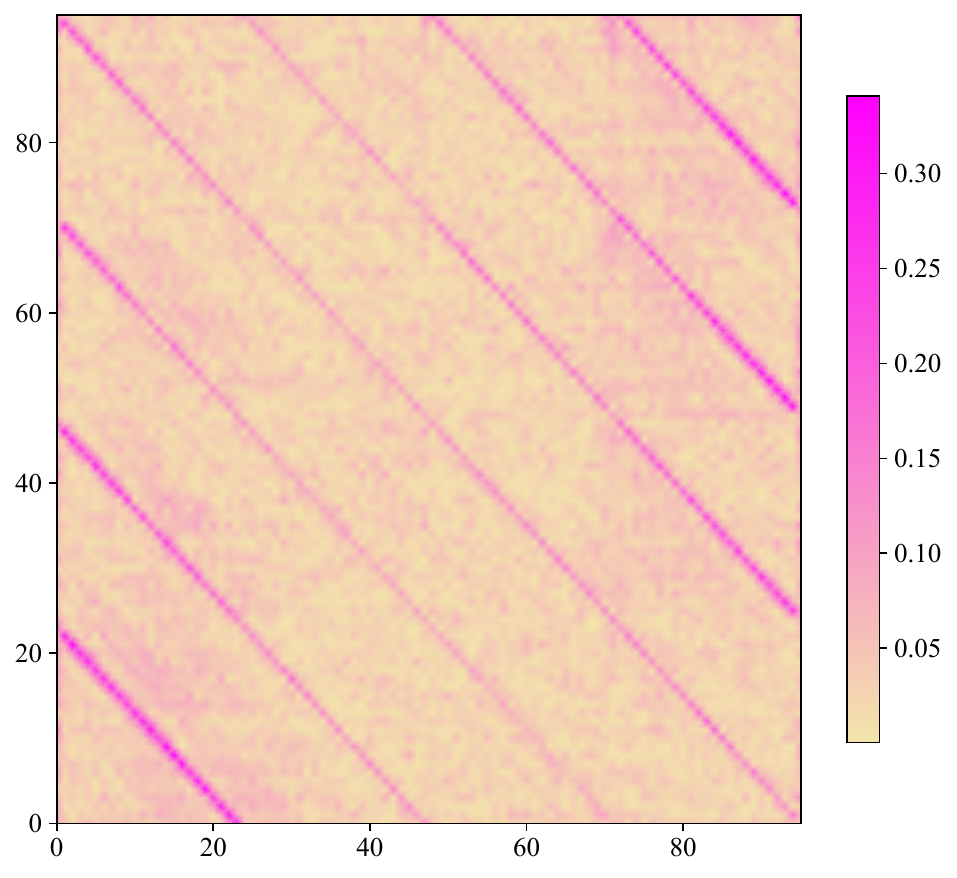}
    \end{minipage}
    \begin{minipage}[t]{0.16\linewidth}
        \centering
        \includegraphics[width=1\linewidth]{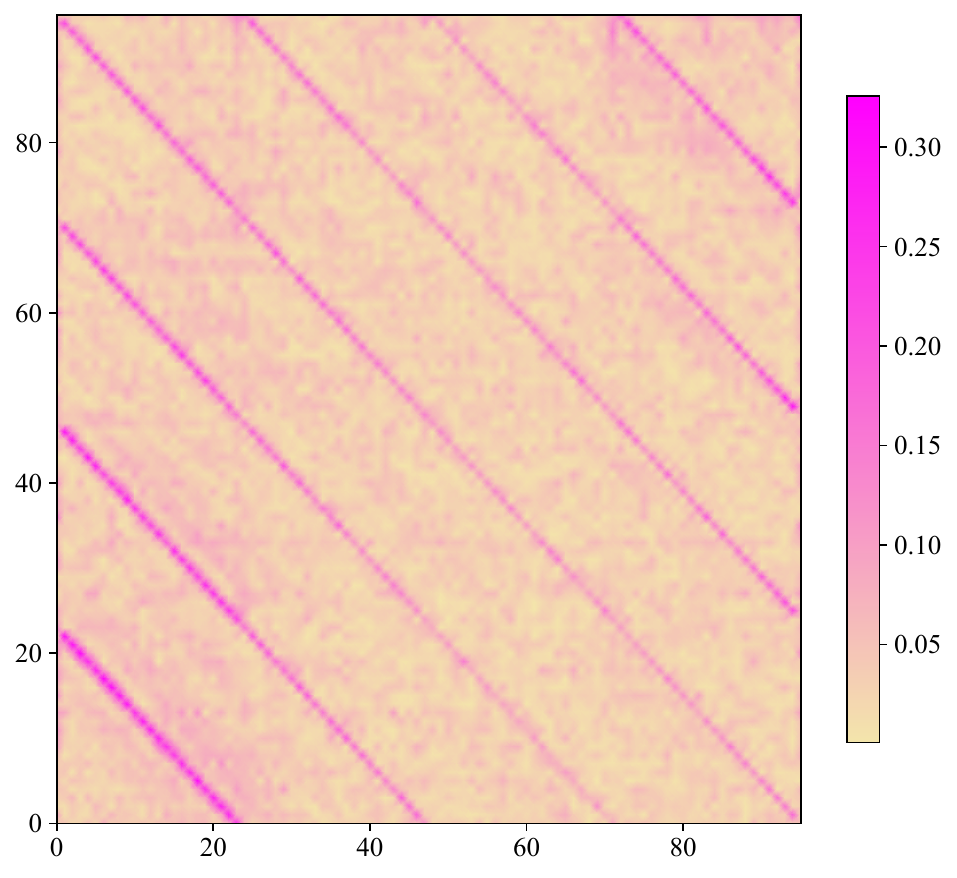}
    \end{minipage}
    \begin{minipage}[t]{0.16\linewidth}
        \centering
        \includegraphics[width=1\linewidth]{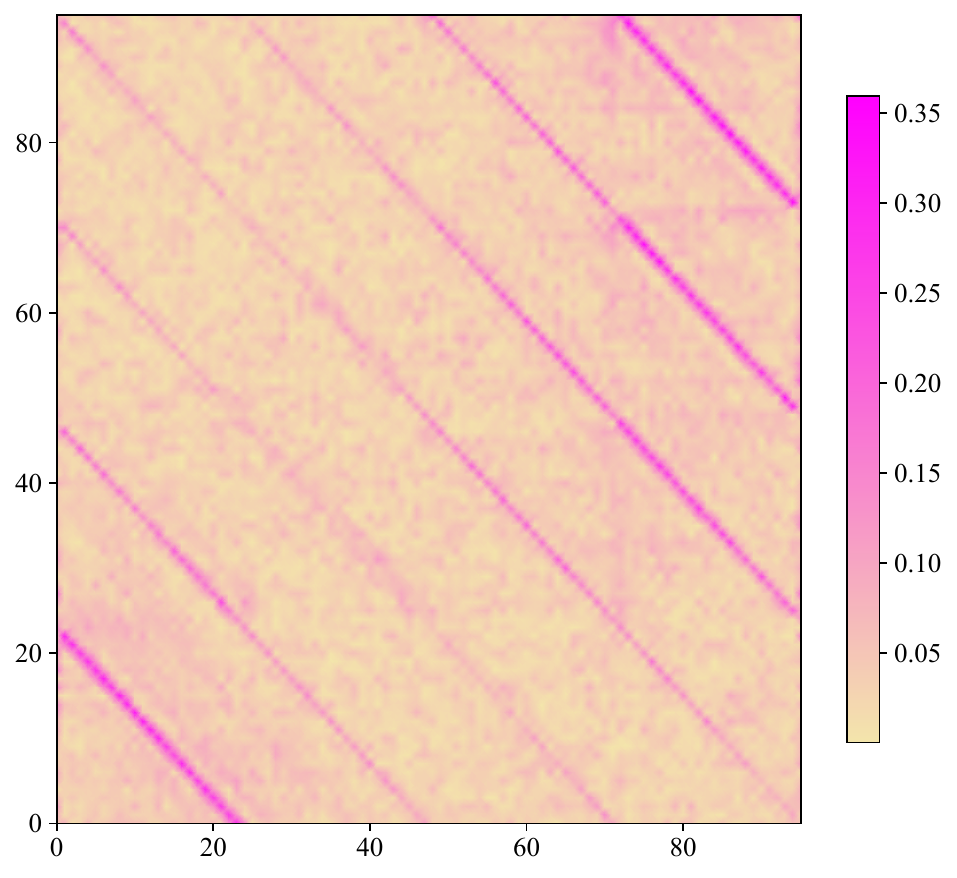}
    \end{minipage}
    \begin{minipage}[t]{0.16\linewidth}
        \centering
        \includegraphics[width=1\linewidth]{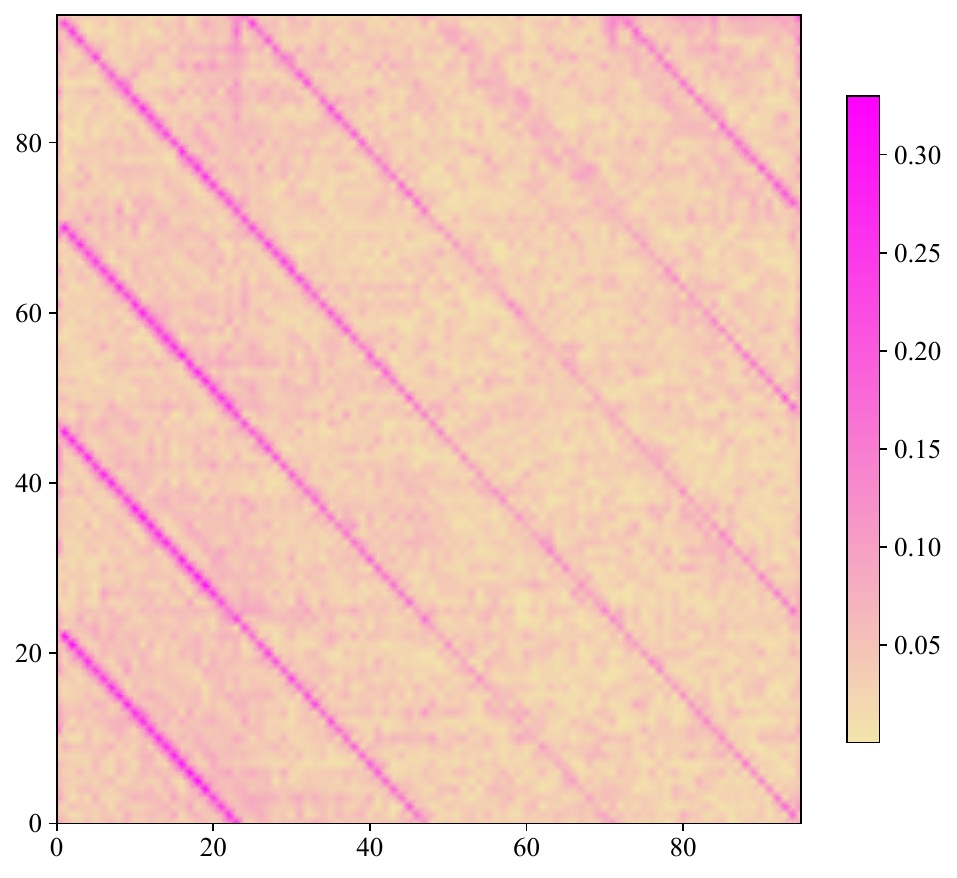}
    \end{minipage}
    \begin{minipage}[t]{0.16\linewidth}
        \centering
        \includegraphics[width=1\linewidth]{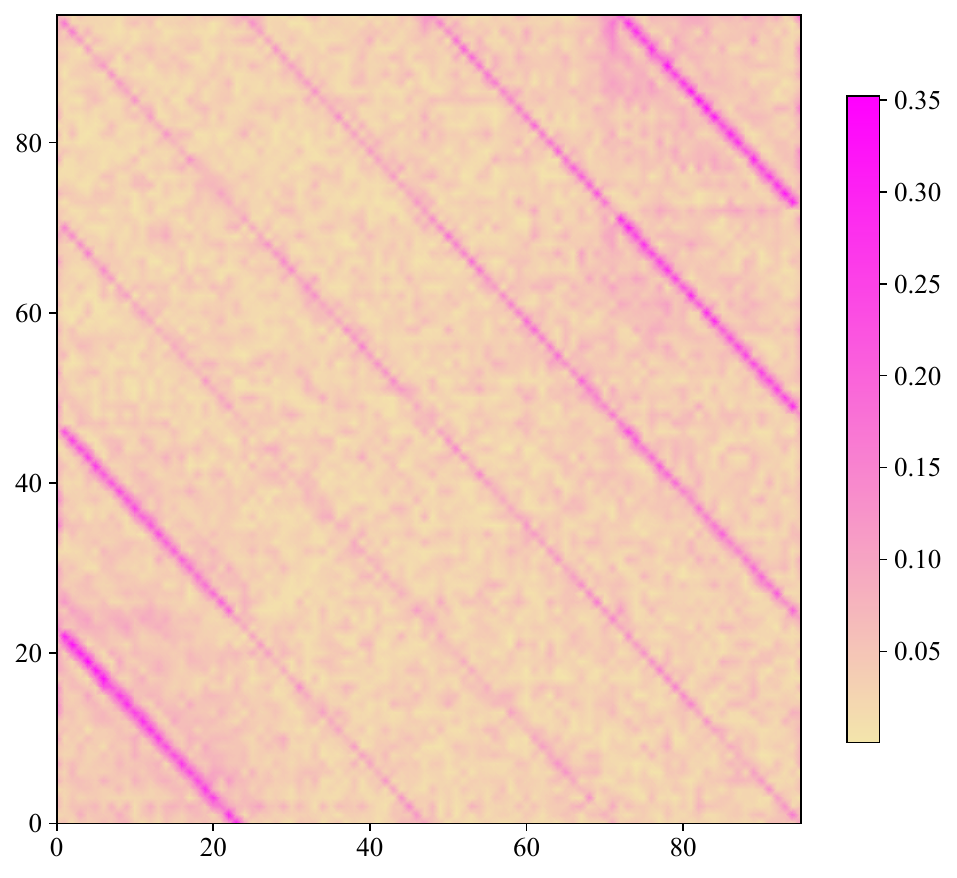}
    \end{minipage}
    
    \qquad
    
    \begin{minipage}[t]{0.16\linewidth}
        \centering
        \includegraphics[width=1\linewidth]{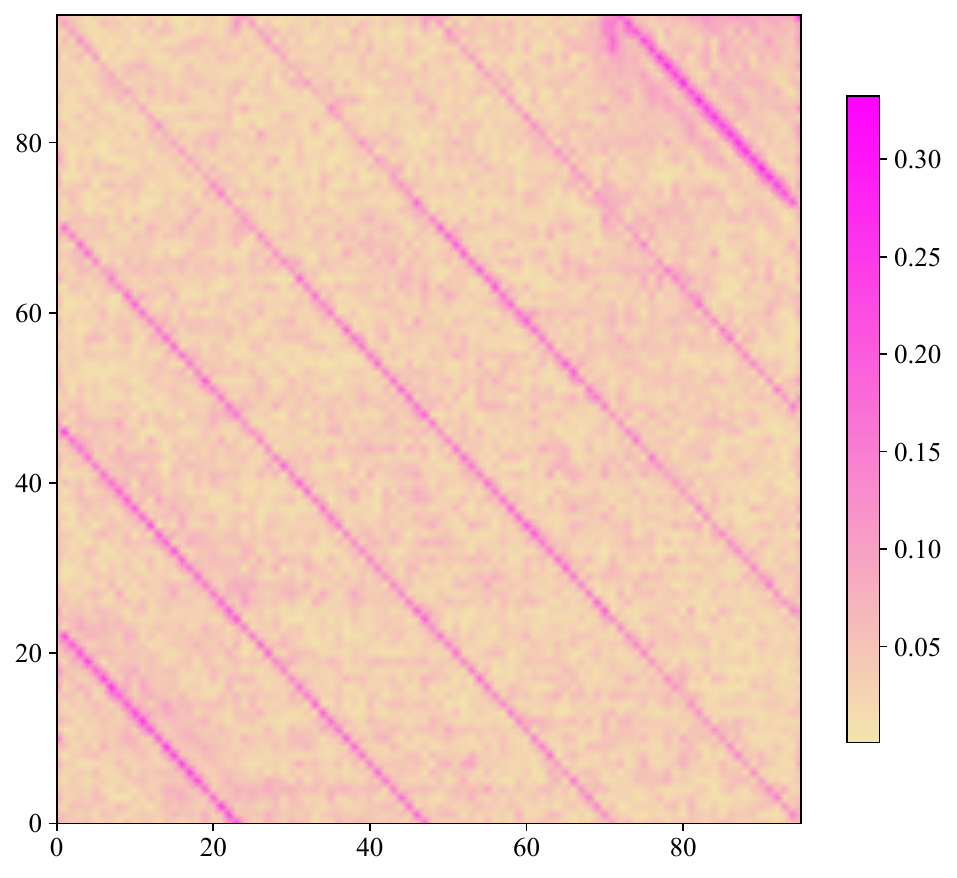}
    \end{minipage}
    \begin{minipage}[t]{0.16\linewidth}
        \centering
        \includegraphics[width=1\linewidth]{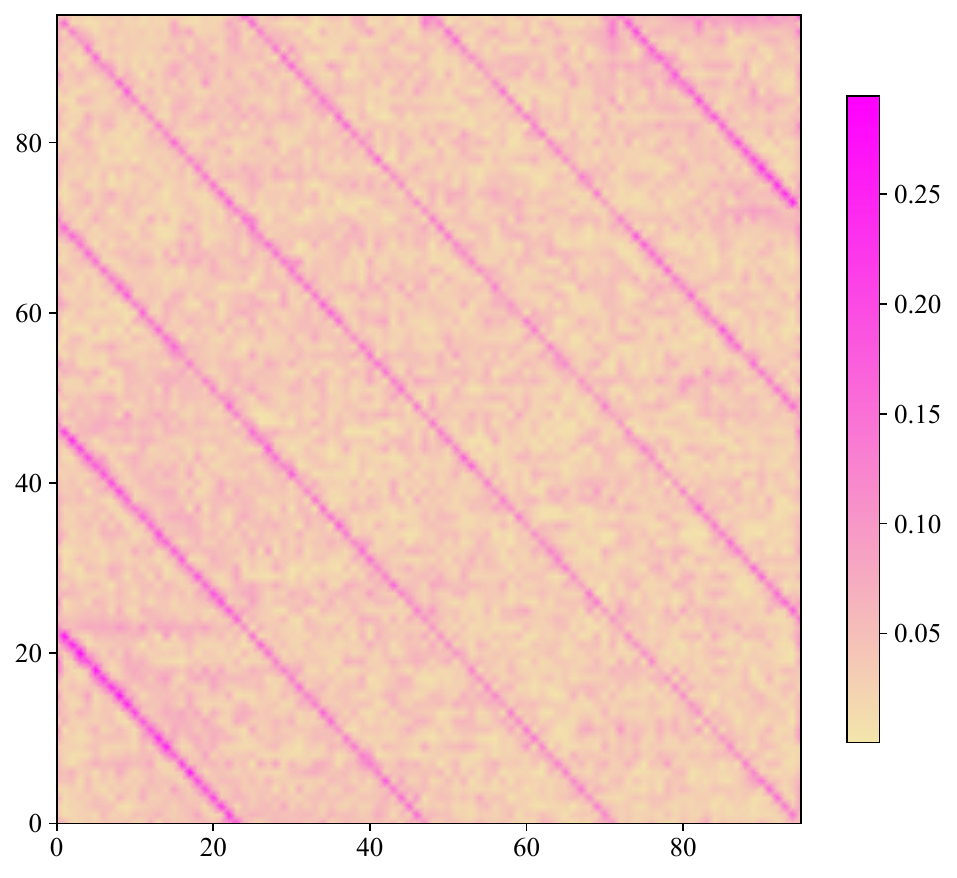}
    \end{minipage}
    \begin{minipage}[t]{0.16\linewidth}
        \centering
        \includegraphics[width=1\linewidth]{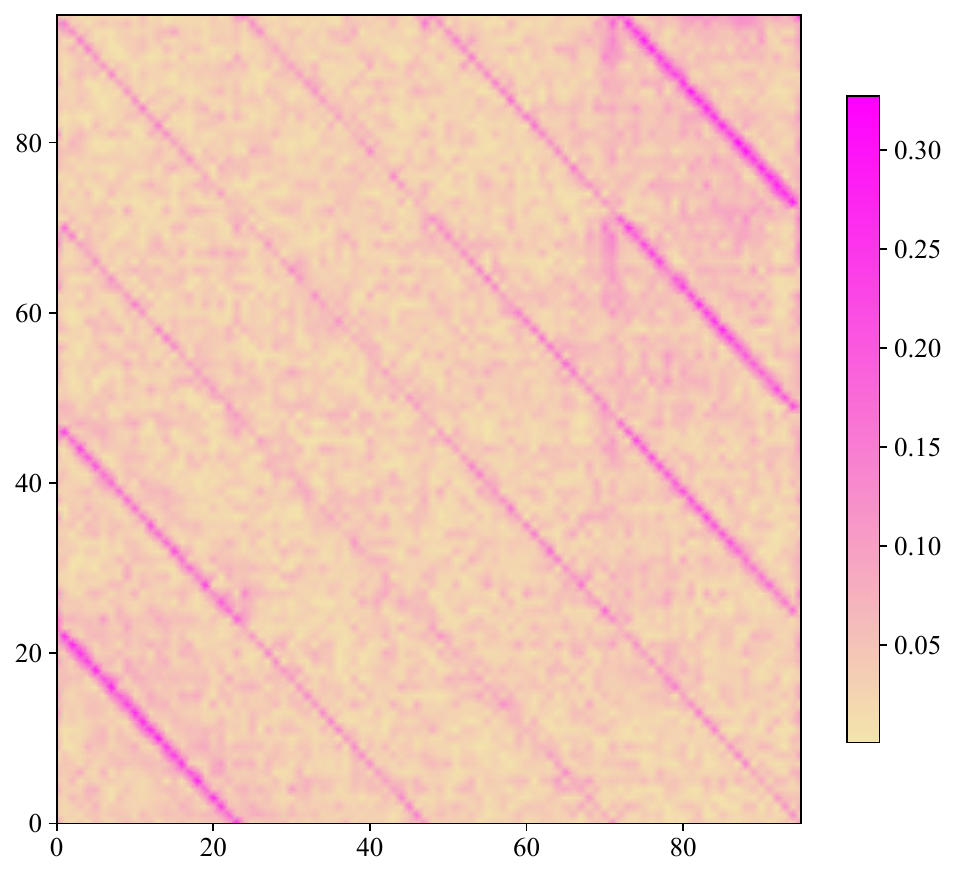}
    \end{minipage}
    \begin{minipage}[t]{0.16\linewidth}
        \centering
        \includegraphics[width=1\linewidth]{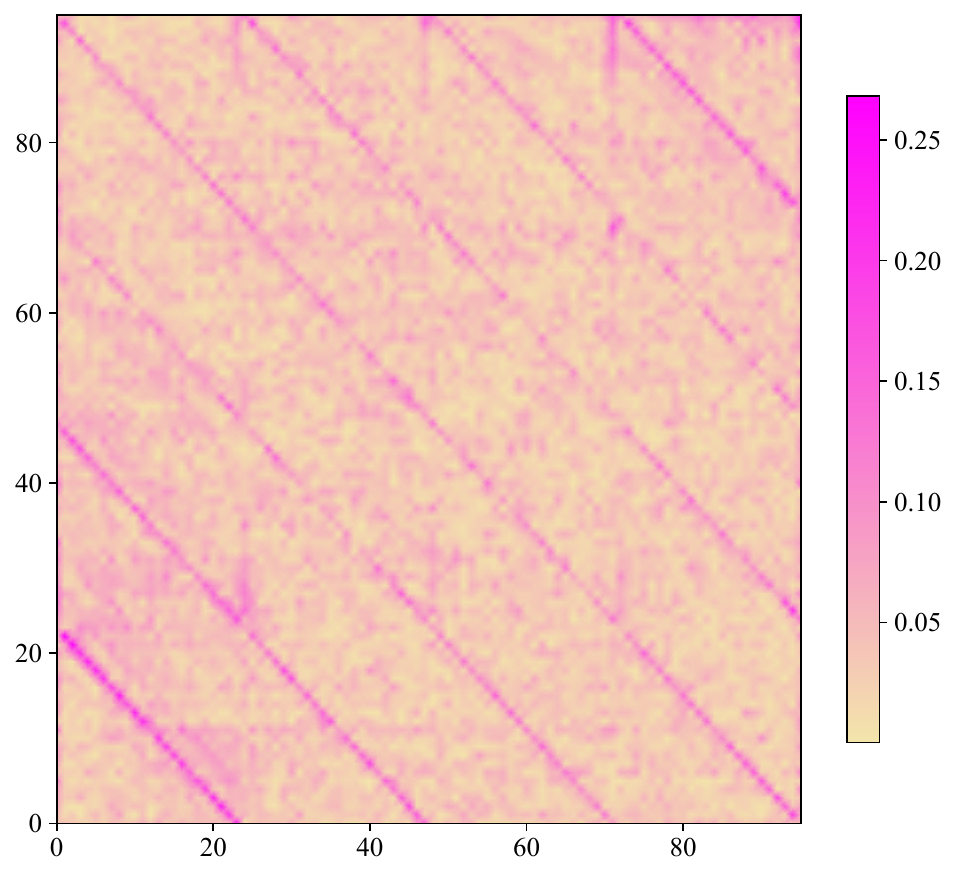}
    \end{minipage} 
     \begin{minipage}[t]{0.16\linewidth}
        \centering
        \includegraphics[width=1\linewidth]{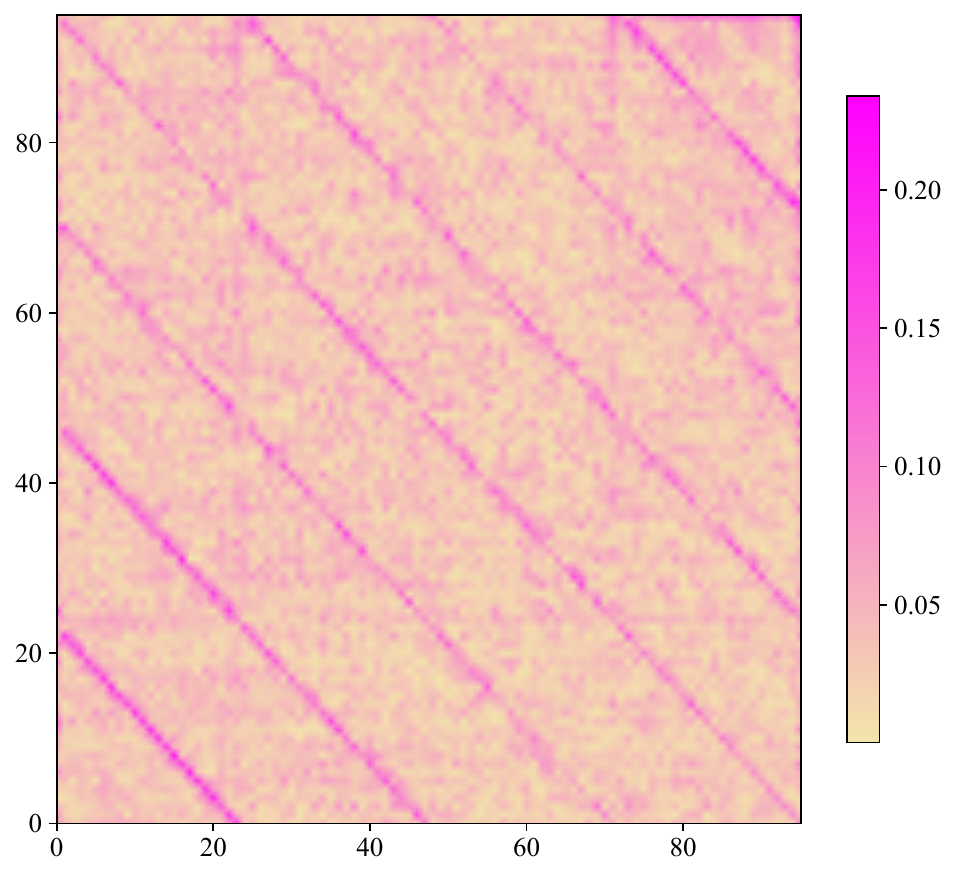}
    \end{minipage}
    \caption{Visualization of the 10 complex linear weights' modulus in CATS-Linear, from left to right. The downward-sloping lines indicate that data periodicity induces corresponding periodicity in model weights.}
    \label{heatmap-fig3}
\end{figure*}

\begin{table}[htb]\rmfamily
    \caption{Prediction errors under different $K$.}
    \label{hyperK-tab4}
    \setlength{\tabcolsep}{0.78mm}
    \begin{small}
    \centering
    \begin{tabular}{c|c|cccccccc}
        \toprule
        \multirow{2}{*}{\rotatebox{90}{Data}}
        &  & \multicolumn{2}{c}{$K$=5} & \multicolumn{2}{c}{$K$=10} & \multicolumn{2}{c}{$K$=20} & \multicolumn{2}{c}{$K$=40} \\
        &  & MSE & MAE & MSE & MAE & MSE & MAE & MSE & MAE \\

        \midrule
        \multirow{4}{*}{\rotatebox{90}{Electricity}} 
        & 96   & 0.142 & 0.236 & \textbf{0.140} & \textbf{0.234} & \textbf{0.140} & 0.235 & 0.141 & 0.235  \\
        & 192  & 0.154 & 0.247 & \textbf{0.153} & \textbf{0.247} & \textbf{0.153} & \textbf{0.247} & \textbf{0.153} & \textbf{0.247} \\
        & 336  & 0.169 & 0.263 & 0.168 & 0.262 & \textbf{0.167} & \textbf{0.261} & \textbf{0.167} & 0.262  \\
        & 720  & 0.209 & 0.295 & \textbf{0.208} & \textbf{0.294} & \textbf{0.208} & 0.296 & \textbf{0.208} & 0.296  \\
        \midrule
        \multicolumn{2}{c|}{Sum} & 0.674 & 1.041 & 0.669 & \textbf{1.037} & \textbf{0.668} & 1.039 & 0.669 & 1.040 \\
        
        \midrule
        \multirow{4}{*}{\rotatebox{90}{ETTm1}} 
         & 96  & 0.293 & 0.338 & \textbf{0.287} & \textbf{0.333} & 0.288 & 0.336 & 0.288 & 0.336  \\
        & 192  & 0.336 & 0.363 & 0.328 & 0.360 & \textbf{0.326} & \textbf{0.359} & \textbf{0.326} & \textbf{0.359}  \\
        & 336  & 0.370 & 0.384 & \textbf{0.364} & \textbf{0.381} & 0.366 & \textbf{0.381} & 0.365 & 0.382  \\
        & 720  & \textbf{0.423} & 0.416 & 0.424 & 0.416 & \textbf{0.423} & \textbf{0.415} & 0.424 & 0.416  \\
        \midrule
        \multicolumn{2}{c|}{Sum} & 1.422 & 1.501 & \textbf{1.403} & \textbf{1.490} & \textbf{1.403} & 1.491 & \textbf{1.403} & 1.493 \\ 
        \bottomrule
    \end{tabular}
    \end{small}
    \vskip -0.45in
\end{table}

\begin{table}[htb]\rmfamily
    \caption{The improvements of CACI on PatchTST and Tide. L=336 for PatchTST and L=720 for TiDE.}
    \label{Ablation-tab5}
    \setlength{\tabcolsep}{0.8mm}
    \begin{small}
    \centering
    \begin{tabular}{c|c|cccccccc}
        \toprule
        \multirow{2}{*}{\rotatebox{90}{Data}}
        &  & \multicolumn{2}{c}{PatchTST} & \multicolumn{2}{c}{+CACI} & \multicolumn{2}{c}{TiDE} & \multicolumn{2}{c}{+ CACI} \\
        & & MSE & MAE & MSE & MAE & MSE & MAE & MSE & MAE \\
        
        \midrule
        \multirow{4}{*}{\rotatebox{90}{Weather}} 
        & 96   & 0.192 & 0.231 & 0.183 & 0.224 & 0.204 & 0.263 & 0.182 & 0.251 \\
        & 192  & 0.233 & 0.262 & 0.183 & 0.269 & 0.237 & 0.295 & 0.211 & 0.278 \\
        & 336  & 0.276 & 0.294 & 0.267 & 0.291 & 0.340 & 0.265 & 0.315 & 0.251 \\
        & 720  & 0.357 & 0.351 & 0.348 & 0.338 & 0.355 & 0.396 & 0.332 & 0.378 \\
        \midrule
        \multicolumn{2}{c}{Avg.} &  0.264 &  0.285 &  0.246 &  0.281 &  0.284 &  0.305 &  0.260 &  0.289 \\

        \midrule
        \multirow{4}{*}{\rotatebox{90}{ETTh1}} 
         & 96   & 0.462 & 0.443 & 0.432 & 0.431 & 0.475 & 0.462 & 0.441 & 0.437 \\
        & 192  & 0.501 & 0.466 & 0.486 & 0.454 & 0.520 & 0.487 & 0.481 & 0.472  \\
        & 336  & 0.545 & 0.498 & 0.529 & 0.487 & 0.569 & 0.518 & 0.538 & 0.501  \\
        & 720  & 0.548 & 0.503 & 0.530 & 0.493 & 0.602 & 0.563 & 0.585 & 0.543  \\
        \midrule
        \multicolumn{2}{c}{Avg.} &  0.514 &  0.478 &  0.494 &  0.466 &  0.542 &  0.508 &  0.511 &  0.488 \\
        \bottomrule
    \end{tabular}
    \end{small}
\end{table}

In this section, we experimentally analyze the contribution of each component in CATS-Linear. The ablation study results are presented in Table \ref{Ablation-tab3}. RevIN introduces two affine transformation parameters into instance normalization. To investigate RevIN's role, we replaced it with standard instance normalization. The results indicate that RevIN provides only marginal improvements compared to instance normalization. Beyond the table, we observe that removing instance normalization entirely causes significant performance degradation, demonstrating that instance normalization—which eliminates scale differences across dimensions—is a critical step.

In CATS-Linear without TSLinear, we substitute TSLinear with a linear layer. This modification leads to a slight increase in prediction error. For CATS-Linear without CACI, we remove CACI and adopt a one-channel method, resulting in a significant error increase. Additionally, as shown in Table \ref{main-tab1}, CATS-Linear significantly outperforms the channel-independent model DLinear, further validating the efficacy of CACI.

\textbf{Hyperparameter K}. In Table \ref{hyperK-tab4}, we investigate the impact of hyperparameter $K$ on the experimental results of two datasets. The findings demonstrate that setting the category count below 10 may lead to slight performance degradation. However, once $K$ exceeds 10, even the 321-dimensional Electricity dataset achieves results comparable to those obtained with $K=20$ or $K=40$. Consequently, CACI reduces the computational complexity of Channel-Independence from $\mathcal{O}(D)$ to $\mathcal{O}(1)$ with respect to $D$. Figure \ref{heatmap-fig3} illustrates the weight values of the complex mapping for CATS-Linear when $K=10$.

In Table \ref{Ablation-tab5}, we investigate the effect of employing CACI as a channel design on the Transformer model PatchTST and the MLP model TiDE. CACI achieves MSE reductions of 7.30\% and 3.89\% for PatchTST on the Weather and ETTh1 datasets, respectively. For TiDE, the reductions are 8.45\% and 5.72\%. This suggests that CACI is a generalizable method adaptable to various network architectures.

\section{Conclusion}
This paper systematically summarizes and analyzes existing channel design methodologies, proposing the novel Classification Auxiliary Channel-Independence (CACI) framework to address their limitations. CACI not only reduces complexity but also enhances forecasting performance. Concurrently, we refine feature decomposition in DLinear and integrate it with CACI to establish the new linear model CATS-Linear. Comprehensive forecasting and ablation studies demonstrate that CATS-Linear delivers efficient, accurate, and tuning-free predictions. 

\clearpage
\clearpage

\bibliography{AS}

\begin{thebibliography}{}

\bibitem[Bach, 2024]{bach}
Bach, F. (2024).
\newblock {\em Learning theory from first principles}.
\newblock MIT press.

\bibitem[Borovykh et~al., 2017]{borovykh2017conditional}
Borovykh, A., Bohte, S., and Oosterlee, C.~W. (2017).
\newblock Conditional time series forecasting with convolutional neural networks.
\newblock {\em arXiv preprint arXiv:1703.04691}.

\bibitem[Box and Jenkins, 1968]{1968arima}
Box, G.~E. and Jenkins, G.~M. (1968).
\newblock Some recent advances in forecasting and control.
\newblock {\em Journal of the Royal Statistical Society. Series C (Applied Statistics)}, 17(2):91--109.

\bibitem[Box et~al., 2015]{box2015arima}
Box, G.~E., Jenkins, G.~M., Reinsel, G.~C., and Ljung, G.~M. (2015).
\newblock {\em Time series analysis: forecasting and control}.
\newblock John Wiley \& Sons.

\bibitem[Chen et~al., 2024]{chen24}
Chen, J., Lenssen, J.~E., Feng, A., Hu, W., Fey, M., Tassiulas, L., Leskovec, J., and Ying, R. (2024).
\newblock From similarity to superiority: Channel clustering for time series forecasting.
\newblock {\em Advances in Neural Information Processing Systems}, 37:130635--130663.

\bibitem[Cleveland, 1990]{st90}
Cleveland, S. (1990).
\newblock A seasonal-trend decomposition procedure based on loess (with discussion).
\newblock {\em Journal of Office Statistics}, 6(3).

\bibitem[Das et~al., 2023]{2023tide}
Das, A., Kong, W., Leach, A., Mathur, S.~K., Sen, R., and Yu, R. (2023).
\newblock Long-term forecasting with tide: Time-series dense encoder.
\newblock {\em Transactions on Machine Learning Research}.

\bibitem[Ekambaram et~al., 2023]{ekambaram2023tsmixer}
Ekambaram, V., Jati, A., Nguyen, N., Sinthong, P., and Kalagnanam, J. (2023).
\newblock Tsmixer: Lightweight mlp-mixer model for multivariate time series forecasting.
\newblock In {\em Proceedings of the 29th ACM SIGKDD conference on knowledge discovery and data mining}, pages 459--469.

\bibitem[Franceschi et~al., 2019]{franceschi2019unsupervised}
Franceschi, J.-Y., Dieuleveut, A., and Jaggi, M. (2019).
\newblock Unsupervised scalable representation learning for multivariate time series.
\newblock {\em Advances in neural information processing systems}, 32.

\bibitem[Gardner and Everette, 1985]{gardner1985}
Gardner, J. and Everette, S. (1985).
\newblock Exponential smoothing: The state of the art.
\newblock {\em Journal of forecasting}, 4(1):1--28.

\bibitem[Genet and Inzirillo, 2024]{genet2024temporal}
Genet, R. and Inzirillo, H. (2024).
\newblock A temporal linear network for time series forecasting.
\newblock {\em arXiv preprint arXiv:2410.21448}.

\bibitem[Graves, 2012]{graves2012lstm}
Graves, A. (2012).
\newblock Long short-term memory.
\newblock {\em Supervised sequence labelling with recurrent neural networks}, pages 37--45.

\bibitem[Han et~al., 2024]{han24}
Han, L., Ye, H.-J., and Zhan, D.-C. (2024).
\newblock The capacity and robustness trade-off: Revisiting the channel independent strategy for multivariate time series forecasting.
\newblock {\em IEEE Transactions on Knowledge and Data Engineering}, 36(11):7129--7142.

\bibitem[Holt, 2004]{HOLT04}
Holt, C.~C. (2004).
\newblock Forecasting seasonals and trends by exponentially weighted moving averages.
\newblock {\em International Journal of Forecasting}, 20(1):5--10.

\bibitem[Ilbert et~al., 2024]{ilbert2024analysing}
Ilbert, R., Tiomoko, M., Louart, C., Odonnat, A., Feofanov, V., Palpanas, T., and Redko, I. (2024).
\newblock Analysing multi-task regression via random matrix theory with application to time series forecasting.
\newblock {\em Advances in Neural Information Processing Systems}, 37:115021--115057.

\bibitem[Kim et~al., 2022]{1}
Kim, T., Kim, J., Tae, Y., Park, C., Choi, J.-H., and Choo, J. (2022).
\newblock Reversible instance normalization for accurate time series forecasting against distribution shift.
\newblock In {\em International Conference on Learning Representations}.

\bibitem[Kingma and Ba, 2014]{kingma2014adam}
Kingma, D.~P. and Ba, J. (2014).
\newblock Adam: A method for stochastic optimization.
\newblock {\em arXiv preprint arXiv:1412.6980}.

\bibitem[Lai et~al., 2018]{lai2018lstnet}
Lai, G., Chang, W.-C., Yang, Y., and Liu, H. (2018).
\newblock Modeling long-and short-term temporal patterns with deep neural networks.
\newblock In {\em The 41st international ACM SIGIR conference on research \& development in information retrieval}, pages 95--104.

\bibitem[Li et~al., 2024a]{li2024vlinear}
Li, C., Xiao, B., and Yuan, Q. (2024a).
\newblock Vlinear: Enhanced linear complexity time series forecasting model.
\newblock {\em Intelligent Data Analysis}, page 1088467X241303376.

\bibitem[Li et~al., 2019]{li2019enhancing}
Li, S., Jin, X., Xuan, Y., Zhou, X., Chen, W., Wang, Y.-X., and Yan, X. (2019).
\newblock Enhancing the locality and breaking the memory bottleneck of transformer on time series forecasting.
\newblock {\em Advances in neural information processing systems}, 32.

\bibitem[Li et~al., 2024b]{li24}
Li, Z., Qi, S., Li, Y., and Xu, Z. (2024b).
\newblock Revisiting long-term time series forecasting: An investigation on affine mapping.

\bibitem[Li et~al., 2023]{23mts}
Li, Z., Rao, Z., Pan, L., and Xu, Z. (2023).
\newblock Mts-mixers: Multivariate time series forecasting via factorized temporal and channel mixing.
\newblock {\em arXiv preprint arXiv:2302.04501}.

\bibitem[Liu et~al., 2022]{liu2022pyraformer}
Liu, S., Yu, H., Liao, C., Li, J., Lin, W., Liu, A.~X., and Dustdar, S. (2022).
\newblock Pyraformer: Low-complexity pyramidal attention for long-range time series modeling and forecasting.
\newblock In {\em \# PLACEHOLDER\_PARENT\_METADATA\_VALUE\#}.

\bibitem[Liu et~al., 2024]{liu24}
Liu, Y., Hu, T., Zhang, H., Wu, H., Wang, S., Ma, L., and Long, M. (2024).
\newblock itransformer: Inverted transformers are effective for time series forecasting.
\newblock In {\em The Twelfth International Conference on Learning Representations}.

\bibitem[Liu et~al., 2023]{liu2023koopa}
Liu, Y., Li, C., Wang, J., and Long, M. (2023).
\newblock Koopa: Learning non-stationary time series dynamics with koopman predictors.
\newblock {\em Advances in neural information processing systems}, 36:12271--12290.

\bibitem[Luo and Wang, 2024]{luo2024moderntcn}
Luo, D. and Wang, X. (2024).
\newblock Moderntcn: A modern pure convolution structure for general time series analysis.
\newblock In {\em The twelfth international conference on learning representations}, pages 1--43.

\bibitem[Nie et~al., 2023]{27}
Nie, Y., Nguyen, N.~H., Sinthong, P., and Kalagnanam, J. (2023).
\newblock A time series is worth 64 words: Long-term forecasting with transformers.
\newblock In {\em The Eleventh International Conference on Learning Representations}.

\bibitem[Oreshkin et~al., 2020]{nbeats20}
Oreshkin, B.~N., Carpov, D., Chapados, N., and Bengio, Y. (2020).
\newblock N-beats: Neural basis expansion analysis for interpretable time series forecasting.
\newblock In {\em International Conference on Learning Representations}.

\bibitem[Rizvi et~al., 2025]{rizvi2025bridging}
Rizvi, S. T.~H., Kanwal, N., Naeem, M., Cuzzocrea, A., and Coronato, A. (2025).
\newblock Bridging simplicity and sophistication using glinear: A novel architecture for enhanced time series prediction.
\newblock {\em arXiv preprint arXiv:2501.01087}.

\bibitem[Smyl, 2020]{SMYLS2020}
Smyl, S. (2020).
\newblock A hybrid method of exponential smoothing and recurrent neural networks for time series forecasting.
\newblock {\em International Journal of Forecasting}, 36(1):75--85.

\bibitem[Tan et~al., 2024]{tan2024language}
Tan, M., Merrill, M., Gupta, V., Althoff, T., and Hartvigsen, T. (2024).
\newblock Are language models actually useful for time series forecasting?
\newblock {\em Advances in Neural Information Processing Systems}, 37:60162--60191.

\bibitem[Toner and Darlow, 2024]{Toner24}
Toner, W. and Darlow, L.~N. (2024).
\newblock An analysis of linear time series forecasting models.
\newblock In {\em Forty-first International Conference on Machine Learning}, pages 48404--48427.

\bibitem[Wang et~al., 2023]{23micn}
Wang, H., Peng, J., Huang, F., Wang, J., Chen, J., and Xiao, Y. (2023).
\newblock Micn: Multi-scale local and global context modeling for long-term series forecasting.
\newblock In {\em The eleventh international conference on learning representations}.

\bibitem[Wang et~al., 2025a]{wang2025clinear}
Wang, J., Su, X., Huang, Y., Lai, H., Qian, W., and Zhang, S. (2025a).
\newblock Clinear: An interpretable deep time series forecasting model for periodic time series.
\newblock {\em IEEE Internet of Things Journal}.

\bibitem[Wang et~al., 2025b]{25timemixer2}
Wang, S., LI, J., Shi, X., Ye, Z., Mo, B., Lin, W., Shengtong, J., Chu, Z., and Jin, M. (2025b).
\newblock Timemixer++: A general time series pattern machine for universal predictive analysis.
\newblock In {\em The Thirteenth International Conference on Learning Representations}.

\bibitem[Wang et~al., 2024a]{23timemixer}
Wang, S., Wu, H., Shi, X., Hu, T., Luo, H., Ma, L., Zhang, J.~Y., and ZHOU, J. (2024a).
\newblock Timemixer: Decomposable multiscale mixing for time series forecasting.
\newblock In {\em International Conference on Learning Representations}.

\bibitem[Wang et~al., 2024b]{wangcard}
Wang, X., Zhou, T., Wen, Q., Gao, J., Ding, B., and Jin, R. (2024b).
\newblock Card: Channel aligned robust blend transformer for time series forecasting.
\newblock In {\em The Twelfth International Conference on Learning Representations}.

\bibitem[Woo et~al., 2022]{wo22}
Woo, G., Liu, C., Sahoo, D., Kumar, A., and Hoi, S. (2022).
\newblock Etsformer: Exponential smoothing transformers for time-series forecasting.
\newblock {\em arXiv preprint arXiv:2202.01381}.

\bibitem[Wu et~al., 2023]{23timesnet}
Wu, H., Hu, T., Liu, Y., Zhou, H., Wang, J., and Long, M. (2023).
\newblock Timesnet: Temporal 2d-variation modeling for general time series analysis.
\newblock In {\em The Eleventh International Conference on Learning Representations}.

\bibitem[Wu et~al., 2021]{23}
Wu, H., Xu, J., Wang, J., and Long, M. (2021).
\newblock Autoformer: Decomposition transformers with auto-correlation for long-term series forecasting.
\newblock In {\em Advances in Neural Information Processing Systems}, pages 22419--22430.

\bibitem[Xu et~al., 2024]{xu24}
Xu, Z., Zeng, A., and Xu, Q. (2024).
\newblock {FITS}: Modeling time series with 10k parameters.
\newblock In {\em The Twelfth International Conference on Learning Representations}.

\bibitem[Yi et~al., 2024]{yi2024filternet}
Yi, K., Fei, J., Zhang, Q., He, H., Hao, S., Lian, D., and Fan, W. (2024).
\newblock Filternet: Harnessing frequency filters for time series forecasting.
\newblock {\em Advances in Neural Information Processing Systems}, 37:55115--55140.

\bibitem[Yue et~al., 2025]{olinear2025}
Yue, W., Liu, Y., Li, H., Wang, H., Ying, X., Guo, R., Xing, B., and Shi, J. (2025).
\newblock Olinear: A linear model for time series forecasting in orthogonally transformed domain.
\newblock {\em arXiv preprint arXiv:2505.08550}.

\bibitem[Zeng et~al., 2023]{16}
Zeng, A., Chen, M., Zhang, L., and Xu, Q. (2023).
\newblock Are transformers effective for time series forecasting?
\newblock In {\em Proceedings of the AAAI Conference on Artificial Intelligence}, pages 11121--11128.

\bibitem[Zhang and Yan, 2023]{23crossformer}
Zhang, Y. and Yan, J. (2023).
\newblock Crossformer: Transformer utilizing cross-dimension dependency for multivariate time series forecasting.
\newblock In {\em The eleventh international conference on learning representations}.

\bibitem[Zhou et~al., 2021]{18}
Zhou, H., Zhang, S., Peng, J., Zhang, S., Li, J., Xiong, H., and Zhang, W. (2021).
\newblock Informer: Beyond efficient transformer for long sequence time-series forecasting.
\newblock In {\em Proceedings of the AAAI Conference on Artificial Intelligence}, pages 11106--11115.

\bibitem[Zhou et~al., 2022]{zhou22}
Zhou, T., MA, Z., Wang, X., Wen, Q., Sun, L., Yao, T., Yin, W., and Jin, R. (2022).
\newblock Film: Frequency improved legendre memory model for long-term time series forecasting.
\newblock In {\em Advances in Neural Information Processing Systems}, volume~35, pages 12677--12690. Curran Associates, Inc.

\end{thebibliography}

\end{document}